\documentclass[a4paper,fleqn]{cas-sc}
\usepackage{float}
\hypersetup{colorlinks = true,linkcolor = blue,anchorcolor =red,citecolor = blue,filecolor = red,urlcolor = red, pdfauthor=author}

\usepackage{amsthm}
\usepackage{algorithm}
\usepackage{algorithmicx}
\usepackage{algpseudocode}
\usepackage{multicol}

\usepackage{bbm}
\usepackage{subfig}

\DeclareMathOperator*{\argmax}{argmax\,}
\DeclareMathOperator*{\argmin}{argmin\,}

\usepackage[natbib,style=numeric,maxnames=100]{biblatex}
\DeclareNameAlias{sortname}{family-given}
\newtheorem{theorem}{Theorem}

\newtheorem{lemma}[theorem]{Lemma}
\newtheorem{definition}[theorem]{Definition}
\newtheorem{proposition}{Proposition}
\theoremstyle{remark}
\newtheorem{remark}[theorem]{Remark}

\makeatletter
\renewcommand{\Function}[2]{%
  \csname ALG@cmd@\ALG@L @Function\endcsname{#1}{#2}%
  \def\jayden@currentfunction{#1}%
}
\newcommand{\funclabel}[1]{%
  \@bsphack
  \protected@write\@auxout{}{%
    \string\newlabel{#1}{{\jayden@currentfunction}{\thepage}}%
  }%
  \@esphack
}
\makeatother

\usepackage{setspace}

\title{Minimum Description Length Clustering to Measure Meaningful Image Complexity}

\begin{document}

\shorttitle{MDL Clustering for Meaningful Image Complexity}    

\shortauthors{Mahon and Lukasiewicz}  

%
\author[1,3]{Louis Mahon}[type=editor,
       auid=1,
       orcid=0000-0003-0571-4611,
       ]

\cortext[1]{Corresponding author}

\ead{oneillml@tcd.ie}


\author[2,3]{Thomas Lukasiewicz}[orcid=0000-0002-7644-1668]


\affiliation[1]{organization={School of Informatics, University of Edinburgh, UK}}
\affiliation[2]{organization={Institute of Logic and Computation, Vienna University of Technology, Austria}}
\affiliation[3]{organization={Department of Computer Science, University of Oxford, UK}}
\maketitle

\doublespacing

\begin{abstract}
       We present a new image complexity metric. Existing complexity metrics cannot distinguish meaningful content from noise, and give a high score to white noise images, which contain no meaningful information. We use the minimum description length principle to determine the number of clusters and designate certain points as outliers and, hence, correctly assign white noise a low score. 
       The presented method is a step towards humans' ability to detect when data contain a meaningful pattern. It also has similarities to theoretical ideas for measuring meaningful complexity.
       We conduct experiments on seven different sets of images, which show that our method assigns the most accurate scores to all images considered. Additionally, comparing the different levels of the hierarchy of clusters can reveal how complexity manifests at different scales, from local detail to global structure. We then present ablation studies showing the contribution of the components of our method, and that it continues to assign reasonable scores when the inputs are modified in certain ways, including the addition of Gaussian noise and the lowering of the resolution. Code is available at \url{https://github.com/Lou1sM/assemblies}.
\end{abstract}

\begin{graphicalabstract}
\includegraphics[width=\textwidth]{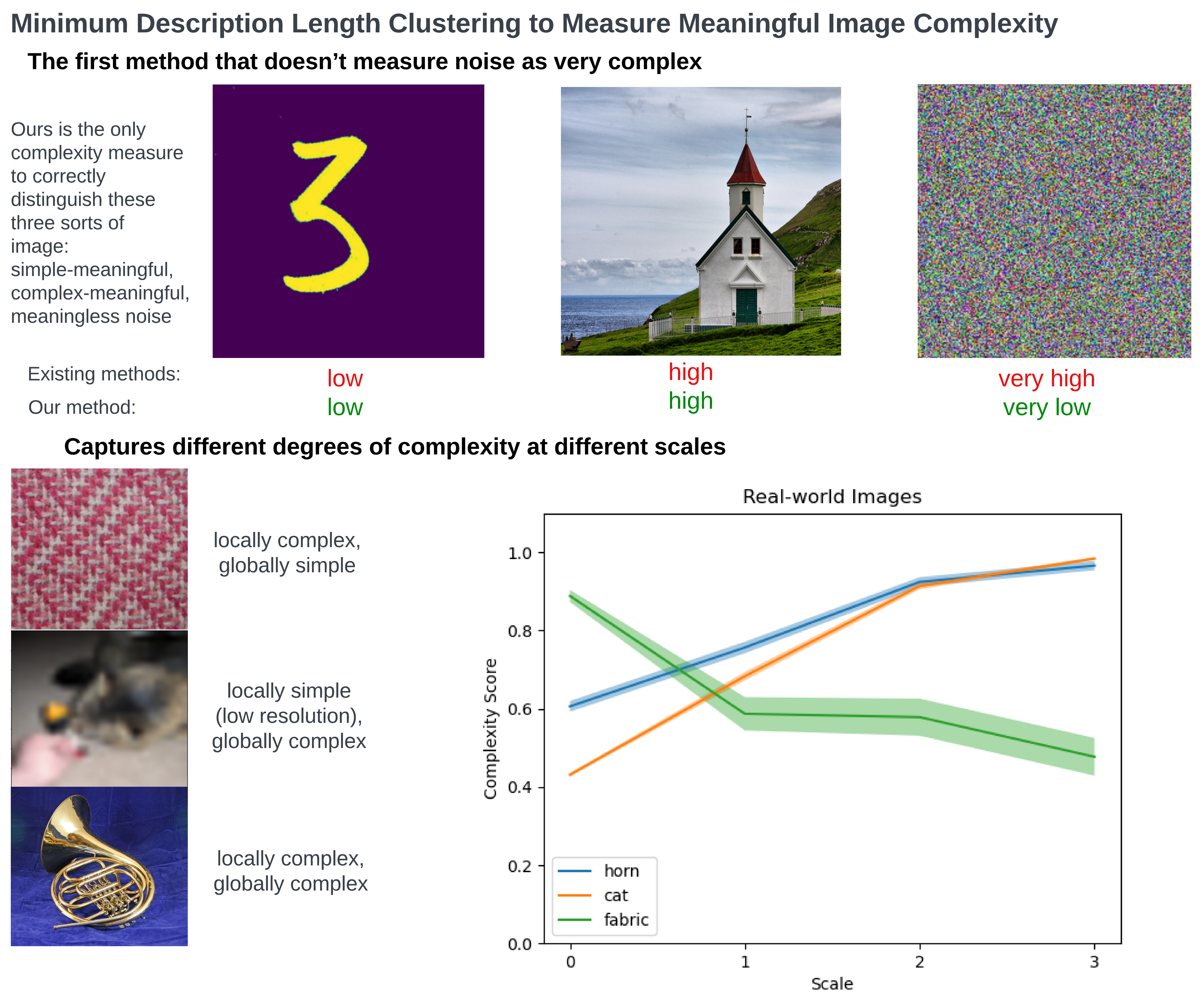}
\end{graphicalabstract}

\section{Introduction} \label{sec:introduction}
Pattern recognition and machine learning typically concern the case where we already know that the given data have some pattern, and we want a method that can automatically discover what the pattern is. In this paper, we address the problem of determining whether the data have any meaningful pattern to begin with, or whether it contains no or only very simple systematic structure, a problem that might be called pattern detection. Humans are highly proficient at recognizing patterns such as words in a speech signal or objects in a video, but even in the absence of explicit recognition, we can often detect when there is a pattern there at all. For example, we hear speech in a foreign language that we do not understand, but can still tell that there is some meaningful structure that could be recognized, unlike ambient city noise or white noise on the radio. Similarly, we may see an abstruse technical diagram and realize that there is something meaningfully complex there, even if we do not know what it is, unlike an image of a blank wall. We study this ability to recognize complexity through the development of a complexity metric. Data with a rich pattern should be scored as high complexity, whatever that pattern is, and unstructured or simply structured data should be scored as low complexity.   

There is unavoidable subjectivity in measuring complexity quantitatively. This is always the case when defining a new metric. We cannot begin the investigation of a complexity metric by defining what complexity is, that would be to put the cart before the horse. Inevitably, the investigation involves exploring what complexity is, not just how to measure it, that is, the definition of complexity and the specification of a complexity metric are two sides of the same thing, the latter is really an instantiation of the former. For example, if one defines complex images as those in which there is a high variation between all the pixel values, then it is natural to use the entropy of pixel values as a complexity metric; or if one defines complex images as those in which nearby pixel intensities tend to be very different from each other, then another metric is the obvious choice (grey-level co-occurrence matrix; see Section \ref{sec:experimental-evaluation}). This renders unavailable the standard blueprint for applied machine learning research of showing that a novel method outperforms existing methods on some quantifiable task or benchmark, because the field does not have such a benchmark for measuring image complexity. What we do have is a vague idea of what complexity is, vague but still powerful and important. The task is to translate this vague idea into something computable.

There are several applications that benefit from being able to measure visual complexity. Remote sensing often gathers large numbers of images, most of which depict nothing interesting, such as empty desert or ocean, but occasionally capture important information, such as the gathering of fauna or sudden change in flora. It is useful to automatically filter out the simple images before manual inspection \citep{falconer2004fractal,sun2006fractal,yang2000analysis}. In the field of psychometrics, there is interest in understanding what humans will find visually interesting or aesthetic, and this often involves a component modelling complexity, e.g., \citep{forsythe2008confounds, carballal2020comparison}. Relatedly, complexity perception influences how humans regard digital interfaces such as graphical websites, and automated complexity measures have been proposed to guide interface design \citep{stickel2010xaos}. Being able to distinguish signal from noise is especially relevant to remote sensing, where images often become corrupted by noise due to the sensing equipment or various post-processing steps
\citep{chioukh2014noise,narayanan2003effects,landgrebe1986noise}. Much work has been done to reduce noise in remote sensing images \citep{chang2016remote,rasti2018noise} and to improve the robustness of image processing methods to noise \citep{huang2020efficient,duan2019noise}.

Most existing techniques for quantifying and measuring image complexity (discussed further in Section~\ref{sec:related-work}) are based on measuring intricacy, the idea being that the more intricate it is and the more dissimilar its parts, the more complex it is. 
This is relatively easy to measure, but it is incomplete for two reasons. Firstly, and most importantly, it does not distinguish between meaningful intricacy (signal) and meaningless intricacy (noise). Using intricacy as a measure of complexity means that a white-noise image, where the pixel values are chosen independently at random, is measured as highly, perhaps even maximally, complex, because there is a high degree of difference between neighbouring pixels. Note that this problem even holds for Kolmogorov complexity, where a standard result is that most bitstrings are almost incompressible, and so, with very high probability, a random bitstring will receive near maximum complexity score. There has been theoretical work to divide the Kolmogorov complexity into meaningful information and noise, using, e.g., `sophistication' \citep{koppel1987complexity,vitanyi2006meaningful} or `effective complexity' \cite{ay2010effective,gell1996information}. Our applied method can be thought of as an instantiation of the high-level idea in these theoretical methods. We return to this comparison in Section~\ref{subsec:comparison-to-theoretical}.

A second disadvantage of variation as a complexity measure is that it cannot capture the fact that images can have a different complexity at different scales. A blurry photograph of a complex scene, for example, is locally simple but globally complex, while a finely-detailed but repetitive pattern is the opposite.

Rather than meaning a high degree of variation, we instead conceive of complexity as `taking a large number of steps to assemble'. An image can be thought of as being built out of pixels, local groups of pixels are combined to form patches, groups of neighbouring patches are combined to form super-patches etc. Quantifying complexity based on the assembly process is the approach taken in the theory of assembly pathways \citep{cronin2006imitation,marshall2019quantifying}, originally for the purpose of quantifying the complexity of molecules to aid in the search for extraterrestrial life \citep{marshall2021identifying,schwieterman2018exoplanet}. The pathway assembly index of an object is the minimum number of combinations needed to produce it from simple parts, where repeated components can be reused without adding to the count. In order to discretize the structure of the image and allow the assembly index to be applied, we employ clustering. For the first level of the hierarchy, we cluster the pixel values and replace them with their cluster index. For higher levels, we cluster the multisets of cluster indices from the level below. 

Another advantage of discretizing is that we can then easily compute entropy. Taking the entropy of a continuous image is difficult, we must use some approximation of differential entropy \citep{hulle2005edgeworth,pichler2022differential}. In our case, however, we are dealing with discrete cluster labels, so we need only compute the entropy of a categorical distribution, which is easy. At each scale (i.e., hierarchy level), we compute the entropy of the multisets of cluster indices across the image to quantify complexity. The total complexity score is the sum of this entropy at each scale. We can also examine the entropy for individual scales to get an indication of the local vs. global complexity in the image: low scales (i.e., small patch sizes) measure local complexity, whereas higher scales capture more global structure (as shown in Section \ref{subsec:complexity-at-scales}). 

At each level of the hierarchy, the cluster indices produced, and hence the complexity score, depend on $K$, the number of clusters in the clustering model. We choose $K$ in a theoretically sound way via the minimum description length (MDL) principle \citep{rissanen1983universal}. MDL says that we should choose the model that can completely represent the given data in the fewest number of bits. Clustering can be interpreted as compression, where we encode each point by its cluster index, along with the residual error of how it differs from the centroid of that cluster. Treating each cluster as a probability distribution, and employing the Kraft-McMillan inequality, we see that the residual error for a point $x$ under the cluster probability distribution $p$ can be represented using $-\log p(x)$ bits. Representing the data under the clustering model takes $-\sum_x\log p(x)$ bits, plus the number of bits to represent the cluster indices and the model itself. Increasing $K$ reduces the average residual error, but increases the size of the indices and the model itself. By MDL, we choose $K$ so as to minimize the total size. MDL is a key component in filtering out noise from our complexity measure. In white noise images, where there is no meaningful or consistent pattern between different points, MDL finds only one cluster, because the small reduction in residual error from encoding more is not worth the extra cost, so the image ends up with a very low complexity score. We both prove this mathematically and observe it empirically.

There are two important similarities between the computational method presented here with human visual perception. The first is hierarchical processing. The visual cortex is divided into five areas, V1-V5. Each takes as input the integrated information output from the previous area, and has progressively larger receptive fields \citep{huff2021neuroanatomy}. This allows humans to perceive each element of a visual scene as composed of smaller elements, e.g., a photograph is composed of man, road, bicycle; the bicycle is composed of wheels, frame, saddle; the wheels are composed of spokes, tyre, valve etc. Similarly, our method processes progressively larger patches of features and passes the output of each level of the hierarchy to the level above as input. 
The approach of treating images as hierarchically structured underpins convolutional neural networks, and has also been leveraged for image segmentation \citep{passat2011interactive}, face recognition \citep{geng2011face}, and  image inpainting \citep{zhang2023fully}.

The second important similarity to human perception is the role of simplicity. Many authors have argued that human perception looks for the simplest interpretation of visual data \citep{chater2005minimum,feldman2016simplicity,sims2016rate}.
Similarly, our use of the minimum description length principle allows us to ignore certain parts of the image and group together other parts in a way that produces the most parsimonious representation overall. This is not a feature of CNNs, but there are some existing works that use MDL clustering for other image processing tasks, such as image segmentation \citep{yang2008unsupervised}, shape modelling \citep{davies2002minimum}, or key-frame extraction \citep{gibson2002visual}. As well being for different tasks, these methods differ from ours in that they do not exclude certain parts of the image as outliers, and are not hierarchical in the sense of passing the output from lower levels to higher levels as input.

The main contributions of this paper are briefly summarized below.
\begin{itemize}  
    \item We propose a novel theoretically sound measure of image complexity and discuss its relationship to ideas in algorithmic information theory.
    \item We test our method empirically on seven image datasets, four public and three synthetic datasets that we created. We show that our method performs as desired in distinguishing images from different datasets. In particular, our method is able to correctly assign a low complexity to white noise, in contrast to existing methods, which assign it a high complexity.
    \item We support these results theoretically by proving that, given normally distributed clusters, MDL will find just a single cluster when the clustering model is fit on white noise, and so our method will assign a low score.
    \item We conduct a further set of experiments, showing how our method can measure complexity at different scales in the image, how it performs when Gaussian noise is added to the image or the resolution of the image is reduced, and how it responds to an increasing fractal dimension of a fractal image. 
\end{itemize}
The rest of this paper is organized as follows. Section \ref{sec:related-work} gives an overview of related work. Section \ref{sec:method} describes our method, and Section \ref{sec:experimental-evaluation} presents our empirical evaluation. Finally, Section \ref{sec:conclusion} summarizes our findings and suggests directions for future work.

\section{Related Work} \label{sec:related-work}

\subsection{Measuring Image Complexity} \label{subsec:img-complexity}
\textbf{Fractal dimension} is a property of curves, which in some sense measures their complexity. 
It can be applied to an image by first binarizing with a threshold, then taking the boundary between white and black pixels as a curve and computing its Minkowski-Bouligand dimension. \citet{lam2002evaluation} explore the use of fractal dimension to measure the complexity of satellite images, and \citet{sun2006fractal} consider the application to remote sensing images more generally. Both also contain a detailed account of methods that use fractal dimension for image complexity. \citet{forsythe2011predicting} compare fractal dimension against human judgements of the complexity and beauty of visual art.

\textbf{File compression ratio} is the ratio between the size of a compressed file under a chosen compression algorithm, and the size of the uncompressed original. \citet{marin2013examining} measure image complexity using the file compression ratio, under two compression algorithms: GIF, which is lossy, and TIFF, which is lossless. The compression ratio was compared to human judgements of complexity, on the International Affective Picture System. It is also used as a complexity measure in \citet{forsythe2011predicting} and by \citet{machado2015computerized}. The former investigate the ability of JPEG-ratio, GIF-ratio, and a novel `perimeter detection' method to predict human judgements of complexity in visual art. The latter explore various combinations of compression algorithms with automated edge detection, and compares the results to human judgements of complexity. The authors find the best results using Sobel and Canny filters, followed by JPEG compression. 

\citet{carballal2020comparison} test the accuracy of various supervised machine learning models of complexity by annotating art and non-art images with human judgements of complexity, then regressing these annotations using a machine learning algorithm that includes feature selection. This was repeated a number of times, and the accuracy of a given feature was taken to be the fraction of times it was selected by the feature selection algorithm.

An alternative method is to use the \textbf{gradient of pixel intensities} across the image. This is the approach taken by \citet{redies2012phog}. The gradient is computed separately for each of the RGB channels, and the gradient at a pixel is taken to be the maximum across the three channels. The average gradient across the entire image is then taken as a measure of complexity. This is again applied to quantifying aesthetic judgements of visual art, this time as part of the Birkhoff-like measure \citep{birkhoff1933aesthetic}, which characterizes beauty as the ratio of order and complexity. 
A final method to consider is \textbf{the Fourier transform}, as used by \citet{khan2022leveraging}. The idea is that the more high-frequency components present in the power spectrum, the more complex the image. The authors investigate using both the mean and the median of the power spectrum, and find best results for the median.
The application in this case is guiding neural architecture search, the claim being that one should first measure the complexity of a given image dataset, and then use the result to inform architecture design.

\subsection{Relation to Other Tasks in Pattern Recognition}
Our method for measuring image complexity begins by assigning a cluster label to each pixel. It can therefore be interpreted as producing a \textbf{segmentation} of the image, by defining a segment as a contiguous set of pixels with the same cluster label. There are several common approaches to image segmentation, such as modified graph-cutting algorithms \citep{peng2011image} or component trees \citep{passat2011interactive}. Among these, the segmentation provided by our method falls into the category that uses only colour and texture information \citep{ilea2011image}, and also relates closely to those methods that use the minimum description length principle \citep{galland2005multi}. We do not directly explore the segmentation quality of our method, but Figure \ref{fig:golf-balls} gives a visual indication of the segments produced.

\textbf{Clustering} is a fundamental task in pattern recognition and machine learning that learns the structure of data in a fully unsupervised way. Current research topics in clustering include the use of deep neural networks such as CNNs \citep{caron2018deep,mahon2021selective} or graph neural networks (GNNs) \citep{fang2023robust}, and exploring alternatives to the standard centroid-based clustering, e.g., density-based clustering \citep{mcinnes2017hdbscan,kumar2016fast}. Our method relates especially to work on reducing the need for hyperparameters such as cluster number \citep{sinaga2020unsupervised,hou2023towards}.

\textbf{Compression} is of strong theoretical and practical interest to pattern recognition, and has been used specifically to measure data complexity by the works described in Section \ref{subsec:img-complexity}. Aside from standard algorithms such as JPEG, common approaches to image compression include deep learning \citep{uchigasaki2023deep,mishra2022deep} and variants of the wavelet transform \citep{haddad2013wave}. By combining clustering with MDL, we treat clustering as a form of compression (see~\cite{mahon2022discrete} for a discussion of clustering as compression) and thus illustrate the connection between compression and data complexity.

\section{Method} \label{sec:method}
This section gives an overview of the minimum description length principle as it is used in our method, then describes our method in detail with the aid of a worked example, and compares our approach, on a high level, to existing theoretical work on meaningful data complexity.

\subsection{Minimum Description Length Patch Clustering}
Our measure of complexity uses a form of clustering based on description length (DL), i.e., the number of bits needed to specify the given data. Description length is relative to an encoding scheme, and via the Kraft-MacMillan inequality, this corresponds to a probability distribution. Specifically, the Kraft-MacMillan inequality says that, under the optimal encoding scheme (optimal in the sense of being shortest on average) of a probability distribution $p(\cdot)$, the description length of a point $x$ is $-\log p(x)$. We model the probability distribution with a Gaussian mixture model (GMM), because (a) we seek a distribution-based clustering model, and a GMM is by far the most commonly used distribution-based clustering model, (b) choosing a GMM is equivalent to simply modelling the distribution within each cluster as normal, and this has theoretical justifications in the central limit theorem and maximization of differential entropy \citep{thomas2006elements}. The description length is therefore relative to the means $\mu = (\mu_i)_{1 \leq i \leq K}$ and the covariances $\Sigma = (\Sigma_i)_{1 \leq i \leq K}$ of this GMM. The probability of a point $x$ under its assigned component of the mixture model $(\mu,\Sigma)$ is given~by
\begin{align} \label{eq:MVN-prob}
    p(x, \mu, \Sigma) = \max_{1 \leq k \leq K} \frac{\exp(-\tfrac{1}{2}(x- \mu_{k})\Sigma_{k}^{-1}(x-\mu_{k}))}{\sqrt{(2 \pi)^{d}|\Sigma_{k}|}}\,,
\end{align}
where $\mu_k$ and $\sigma_k$ are, respectively, the mean and covariance of the $k$th component, and $d$ is the dimensionality of the data. Specifying $x$ under $p$ requires first indexing the cluster to which $x$ belongs and then encoding $x$ under the probability distribution of that cluster, which we refer to as the residual error. The latter was just shown to take $-\log p(x, \mu, \Sigma)$ bits. Similarly, the length of the former depends on the encoding scheme for, and equivalently the probability distribution over, the indices $1, \dots, K$, which can be taken empirically from the data. Specifically, the length of encoding which cluster $x$ belongs to is $-\log {n_k}/{N}$, so the total description length is then 
\begin{align} \label{eq:MVN-DL}
    \min_{1 \leq k \leq K} -\log \tfrac{n_k}{N} + \tfrac{1}{2}(x- \mu_{k})\Sigma_{k}^{-1}(x-\mu_{k}) + \tfrac{1}{2}\log{(2 \pi)^{d}|\Sigma_{k}|}\,,
\end{align}
where $k$ is the index of the cluster that it belongs to, $n_k$ is the number of points belonging to cluster $k$, and $N$ is the total number of data points. As discussed in Section \ref{subsec:comparison-to-theoretical}, we can conceive of the $-\log {n_k}/{N}$ term as the meaningful portion of this description and the remainder as the meaningless portion.

\subsubsection{Differential Description Length}
Because the multivariate normal distributions composing the GMM are continuous probability density functions (pdf), instead of probability mass functions as in the discrete case, it is possible that $p(x, \mu, \Sigma) > 1$. 
Note that this is always a possbility for pdfs, e.g., the univariate Normal distribution 
\[
\mathcal{N}(\mu,\tfrac{1}{5\sqrt{2\pi}})
\]
has the value $5$ at $x=\mu$.
In these cases, the Kraft-MacMillan inequality would seem to suggest that the corresponding encoding scheme can represent $x$ with a strictly negative number of bits, which of course is not possible. The apparent contradiction is resolved by making explicit the precision with which $x$ is to be encoded. Completely specifying any real number is not possible with a finite number of bits, instead one can only specify an extended region $D_x \subset \mathbb{R}^n$, which contains $x$. The number of required bits is then determined by the probability mass inside $D_x$, which is given by 
\begin{gather} \label{eq:discretized-pmf}
    p_m(D_x, \mu, \Sigma) = \int_{D_x} p(z, \mu, \Sigma) dz\,.
\end{gather}
Let $\epsilon$ be the coordinate-wise precision for specifying $x$, i.e., set $D_x$ to be a hypercube of side-length~$\epsilon$. The probability mass in $D_x$ is then approximated as $p(x, \mu, \Sigma) \epsilon^{d}$, giving the description length
\begin{equation} \label{eq:pmf-dl}
-d \log \epsilon -\log(p(x, \mu, \Sigma) - \log n_k/N \,.
\end{equation}
The additional term $-d\log{\epsilon}$ will be higher for smaller $\epsilon$, and will always increase the total description length to be positive even if $-\log(p(x, \mu, \Sigma) < 0$. That it will be large enough to counterbalance $-\log(p(x, \mu, \Sigma)$ is clear from observing that the probability mass in \eqref{eq:discretized-pmf} is never greater than $1$. 

Note that the additional $-d\log{\epsilon}$ term is independent of the pdf itself. Thus, it can be ignored when using MDL and comparing different pdfs (which correspond to different fit clustering models). That is, when invoking the MDL principle, it is sufficient to look only at the term remaining after the $-d\log{\epsilon}$ term has been removed:
\begin{equation} \label{eq:MVN-DDL}
    -\log(p(x, \mu, \Sigma) - \log n_k/N \,.
\end{equation}
We refer to this remaining quantity as the differential description length. We define the differential description length (DDL) to be the negative logarithm of the probability density. It is the continuous analogue of the description length, just as differential entropy is the continuous analogue of entropy. Similarly to differential entropy, DDL can be negative. This happens precisely when the probability density is greater than 1, as just discussed. DDL is related to the description length as follows: for a point $x$ with DDL $D$, the number of bits required to specify it to a precision $\epsilon$ is $\max(\{0,-D - d\log{\epsilon\}})$. 
The max is required to account for the case where the region specified by the precision $\epsilon$ is larger than the interval in which we already know $x$ to lie. For example, if we assume a priori that $x$ is uniformly distributed on $[0,1]$, in which case all points have DDL $0$ under the prior distribution, and then we try to specify to precision $2$, we will end up with 
\[
-D - d\log{\epsilon} = 0 - \log{2} = -1 \,.
\]
Taking the maximum with zero means that, in such cases, we obtain the correct result of $0$.

\subsubsection{Determining Outliers}
As well as choosing the number of clusters (see Section \ref{subsec:determining-num-clusters}), 
we can use the minimum description length (MDL) principle to determine which points are outliers with respect to the given model. An outlier can be defined as one that takes more bits to specify under the model than it does to specify directly, independently of the model. We can always specify (up to finite precision $\epsilon$) any point directly using the same discretizing reasoning as above. First, restrict attention to some bounded region of $\mathbb{R}^n$, which is large enough so that we can assume that it will contain all values the data could have.\footnote{There are several reasonable choices for such a bounded set: the range of values that can be specified using a standard 32-bit float or the hyperrectangle whose sides are the coordinate-wise ranges across the dataset of patches. We find that the exact choice does not affect results. In our implementation, we choose the hypercube whose sides, in each dimension, run from the minimum to the maximum values across all dimensions in the dataset.} Once this bounded region is specified, partition it into a set of small regions--hypercubes with side-length $\epsilon$--and then specify a point $x$ by indexing the unique region that contains $x$. The number of possible regions is 
\[
\left(\frac{a_{max}-a_{min}}{\epsilon}\right)^d\,,
\]
where $d$ is the dimensionality of the data, and $a_{max}$ and $a_{min}$ are the maximum and minimum values, respectively, that appear anywhere in the image. The number of bits to specify a point directly is then 
\begin{align} \label{eq:indexing-epsilon-hypercube}
    \log \left(\frac{a_{max}-a_{min}}{\epsilon}\right)^d = -d \log \epsilon + d \log (a_{max}-a_{min})\,.
\end{align}
Again, we can ignore the precision value $\epsilon$, because it will appear equally in both description length under the model and the description length from indexing the hypercube. Instead, we can use the differential description length. The indexing of the $\epsilon$ hypercube in \eqref{eq:indexing-epsilon-hypercube} is equivalent, when using the differential description length, to using a uniform prior on $[a_{min},a_{max}]^d$. Under such a distribution, the DDL of any point is $d \log (a_{max}-a_{min})$.
Comparing to the DDL under the model, as in \eqref{eq:MVN-DDL}, a point is an outlier iff
\begin{gather}
   -\log(p(x, \mu, \Sigma) - \log \frac{n_k}{N} > d \log (a_{max}-a_{min}) \iff \\
    \iff p(x, \mu, \Sigma)\frac{n_k}{K} < (a_{max}-a_{min})^{-d} \,,
\end{gather}
where, as above, $n_k$ is the number of points assigned to the same cluster as $x$. We  can then define the total DDL of $x$, where $x$ can be specified either directly or using the encoding scheme from the model, as 
\begin{equation} \label{eq:d-definition}
    D(x, \mu, \Sigma) = \min \left(d \log (a_{max}-a_{min}), -\log(p(x, \mu, \Sigma) - \log \frac{n_k}{N} \right)\,.
\end{equation}

\subsubsection{Determining the Number of Clusters} \label{subsec:determining-num-clusters}
For a given set of independent points, $X=(x_i)_{1 \leq i \leq N}$, we have
\begin{equation}
    -\log{p(X)} = -\log{\prod_{i=1}^N p(x_i)} -\sum_{i=1}^N \log{p(x_i)} \,,   
\end{equation}
so the description length of the entire set is the sum of the description lengths of all its points, and the same for the DDL.
The description length of $X$ under the GMM depends on the number of clusters in the GMM, and using the MDL principle, we can determine the optimum number of clusters by regarding `optimum' as meaning `produces the smallest DDL'. 

Let $\mu(X,K), \Sigma(X,K)$ denote the values of $\mu$ and $\Sigma$ with $K$ components, which maximize the probability of $X$:
\begin{gather}
    \mu(X,K), \sigma(X,K) = \argmax_{\mu, \Sigma} \prod_{x \in X} p(x, \mu, \Sigma)\,.
\end{gather}
Finding these optimal parameters means fitting the GMM to the dataset $X$, and can be performed with the usual expectation-maximization algorithm. Denote by $D(X,K)$ the DDL of $X$ under the optimal encoding corresponding to this fit GMM. 
Using $D(\cdot)$ from \eqref{eq:d-definition}, 
we have
\begin{align} \label{eq:dset-dl}
    D(X,K) = \sum_{x \in P(X)} d(x, \mu(X,K), \sigma(X,K))\,.
\end{align}
The value of $D(X,K)$ is the description length of the model itself plus the DDL of $X$ under the model. The former, i.e., the description length of a GMM with $K$ parameters, is, for precision $\epsilon$, given by
\begin{equation} \label{eq:DL-by-K}
    D(K) = Kd\log\left(a_{max}-a_{min}\right)\ + Kd^2\log\left(a_{max}-a_{min}\right)\,.
\end{equation}
Then, the optimal number of clusters $K^*$ is that which minimizes the total description length:
\begin{equation} \label{eq:k-star}
    K^* = \argmin_{1\leq K \leq |X|} D(X,K) + D(K)\,.
\end{equation}
Note that one only needs to consider values of $K$ up to the size of the dataset, as adding more clusters beyond that point can only increase the total description length. In practice, we test only values up to $8$, as fitting GMMs with many clusters becomes expensive and, in our experiments, does not change results. 

\begin{theorem} \label{main-theorem}
When clustering white noise in $[0,1]^m$, using a GMM with $k$ components, the expected DDL of a point is a monotonically increasing function of $k$.
\end{theorem}
See the appendix for a proof. This means that, for white noise, we should expect MDL to select the model with just a single cluster, in which case every point will receive the same cluster label and the resulting entropy will be zero.

Determining the outliers and the number of clusters is relevant to measuring complexity, because it will affect the cluster model that is learnt, and so affect the cluster labels that are assigned and, ultimately, our complexity score.

\subsection{Hierarchical Patch Entropy} \label{subsec:patch-entropy}

\begin{figure*}[t]
    \centering
    \includegraphics[width=0.9\textwidth]{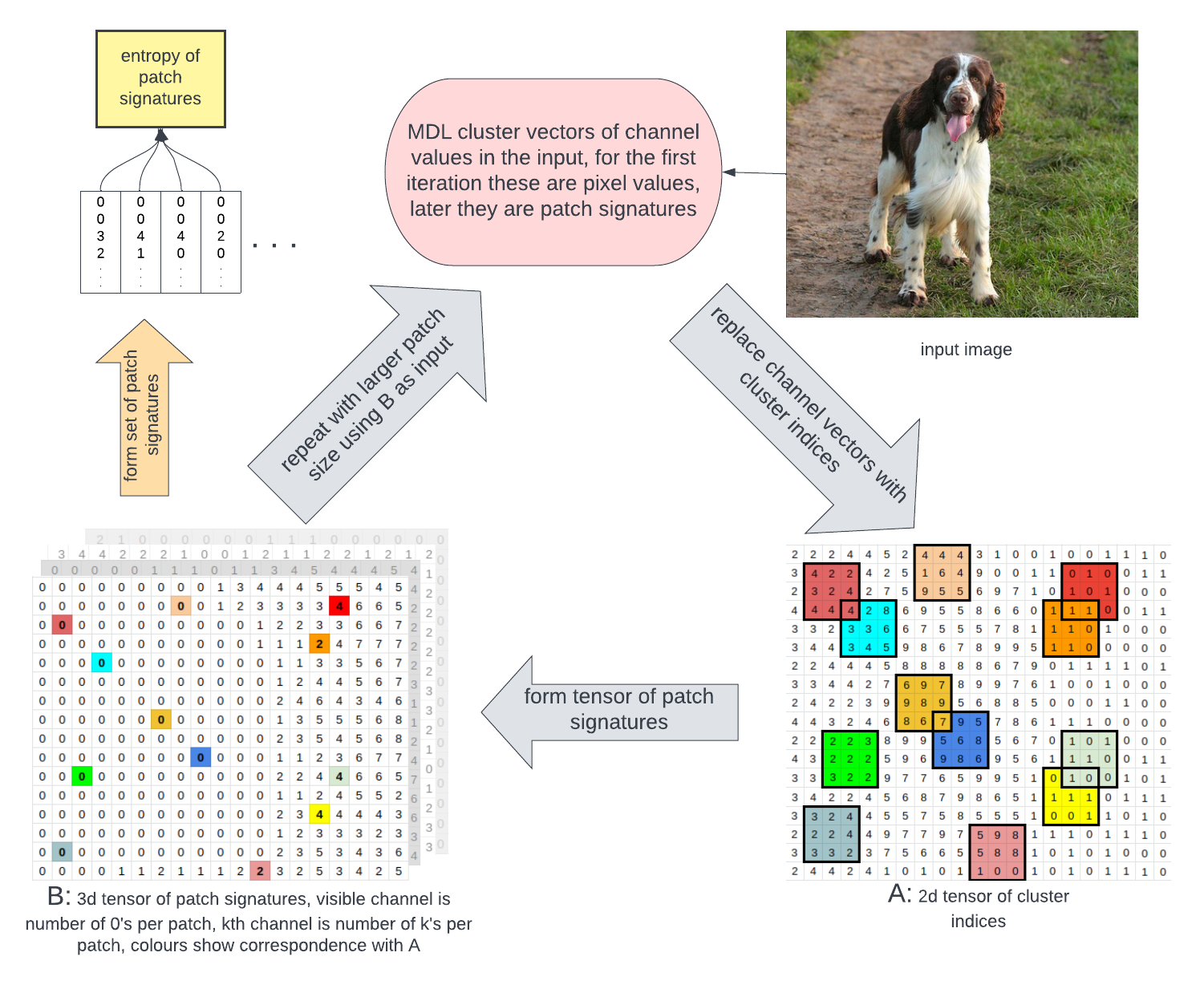}
    \caption{ \small  \small Method for computing the entropy of patch signatures as a measure of complexity. Each patch signature is the multiset of MDL cluster indices that appear there.}
    \label{fig:method-description}
\end{figure*}

The method described in this section is depicted graphically in Figure~\ref{fig:method-description}. At each level of the hierarchy, we begin with a 3d tensor $X$ of shape $(H,W,C)$ and will cluster the vectors of the last dimension; on the first level, this means clustering 3d vectors specifying the colour intensities for each of the three colour channels at each point. Before clustering, the model computes $K^*$ as in \eqref{eq:k-star}, then clusters the last-dimension vectors of $X$ using a mixture model with $K^*$ components. From this clustering, we can form the 2d tensor $A$, of shape $(H,W)$ whose $(i,j)$th entry is the cluster index of the $(i,j)$th pixel in $X$, and $B$, the 3d tensor of shape $(H-m+1,W-m+1,K^*)$ whose $(i,j,k)$th entry is the count of how many times the $k$th cluster appears in the $m \times m$ patch beginning at $(i,j)$ in $A$. 

The patch size $m$ is a user-set parameter. We refer to the vector at location $i,j$ in $B$ as the \emph{signature} of the $(i,j)$th patch. Our measure of entropy at this level is the entropy of the categorical distribution of all signatures that appear in $B$.

As an example of how a patch signature in $B$ is formed from the corresponding patch of cluster indices in $A$, consider the top-left coloured patch in $A$, at the bottom-right of Figure \ref{fig:method-description}. This patch, coloured in dark red, contains three copies of index $2$, one copy of index $3$, five copies of index $4$, and no copies of any other index. Thus, the patch signature is the vector $[0,0,3,1,5,0,0,0,0,0]$. At the bottom-left of Figure \ref{fig:method-description} we see this patch signature is then stored at the corresponding location in $B$,  also in dark red; note the first channel showing $0$, the first element in the patch signature.

To measure complexity at a larger scale, we repeat the above procedure, this time beginning with $B$ instead of $X$. Let subscripts denote the level of the hierarchy, so that $A_i$ and $B_i$ are the tensors formed, as just described, on the $i$th level of the hierarchy. Then, we can say that $B_i$ contains the signatures (i.e., counts vectors) of the patches in $A_i$, and $A_i$ contains the MDL-cluster indices of the last-dimension vectors in $B_{i-1}$. To begin the iteration, $B_0$ is set to $X$, the input image. 

The present implementation computes up to $B_4$, and uses larger patch sizes for each level: $4$, $8$, $16$, and $32$. Note, however, that this is not the same as simply clustering larger patches of an image. What is clustered at each level is the cluster indices from the level below, so is quite different from the input image. The full method is described in Algorithm~\ref{alg:method}. 


\begin{algorithm}
\caption{ \small  \small Algorithm for computing the complexity of an image.} \label{alg:method}
\begin{algorithmic}
    \Function{MDL\_Cluster}{D}
    \State $best\_DL \gets \infty$
    \State $A \gets$ cluster indices of MDL of $D$, initialized randomly
    \For{$K \in \{1,\dots,K\_max\}$}
        \State fit a GMM with $K$ components to $D$
        \State $DL \gets$ differential description length of $D$ under this fit GMM, as per \eqref{eq:DL-by-K} 
        \If{DL < best\_DL}
            \State $A \gets$ cluster indices of $D$ under this fit GMM
            \State $best\_DL \gets DL$
        \EndIf
    \EndFor
    \State \Return $A$
    \EndFunction

    \Function{Signatures\_Entropy}{S}
        \State $bin\_counts \gets$ hash table whose keys are the unique elements in $S$, and whose values are the number of times that element occurs in $S$
        \vspace{3pt}
        \State \Return -$\sum_{b \in bin\_counts} \frac{bin\_counts[x]}{|S|}\log \frac{bin\_counts[x]}{|S|}$
    \EndFunction
        
    \Function{Compute\_Patch\_Signatures}{X,m}
        \State $A \gets$ \Call{MDL\_Cluster}{$X$}
        \State $B \gets$ multisets of cluster indices appearing in all $m \times m$ patches of $A$ (including overlapping)
        \State \Return $B$
    \EndFunction
        
    \Function{Complexity}{X,scales}
        \State $total\_complexity \gets 0$
        \For{$m \in scales$}
            \State $X \gets$ \Call{Compute\_Patch\_Signatures}{$X,m$} 
            \State $total\_complexity \gets total\_complexity + $ \Call{Signatures\_Entropy}{$X$}
        \EndFor
        \State \Return $total\_complexity$
    \EndFunction

\end{algorithmic}
\end{algorithm}

The method begins with the function MDL\_Cluster, which returns the cluster indices of the MDL clustering of each location in the input. The right-hand-side of Figure \ref{fig:golf-balls} shows an example of the output of this function when applied to the image from the left-hand-side of Figure \ref{fig:golf-balls}.

\subsection{Comparison with Theoretical Measures of Meaningful Complexity} \label{subsec:comparison-to-theoretical}
As mentioned in Section \ref{sec:introduction}, previous works have explored, theoretically, how one might divide the algorithmic information of an object into a meaningful portion and a meaningless portion via sophistication \citep{vitanyi2006meaningful,koppel1987complexity} and effective complexity \citep{ay2010effective,gell1996information}. The applied method that we present in this paper shares the same high-level approach to these theoretical ideas, namely, to select the description for our data that has shortest overall length, and then, within that shortest description, select the size of the meaningful portion as a measure of the data complexity. 

We assume that we have some way of distinguishing meaningful vs. meaningless descriptions. In our case, meaningful descriptions correspond to assignments of cluster labels to different parts of the image, and have length given by the first term in \eqref{eq:MVN-DL}; the meaningless descriptions correspond to the residual error in specifying a point exactly given its cluster label, as per the second two terms in \eqref{eq:MVN-DL} along with the specification of outliers as per \eqref{eq:indexing-epsilon-hypercube}. Sophistication and effective complexity, on the other hand, characterize the meaningful portion as a description of a set of which the given data is a typical member, and the meaningless portion corresponds to selecting the given data from within that set. Let $\mathcal{S}$ and $\mathcal{R}$ denote, respectively, the sets of all possible meaningful and meaningless descriptions. Given data $X$, we write 
\begin{equation}
 D_0, \dots, D_n \vdash X,\, \text{ where } D_i \in \mathcal{S} \cup \mathcal{R}\,, \forall 1 \leq i \leq n
\end{equation}
to mean that descriptions $D_0, \dots, D_n$ together perfectly describe $X$. We might try to characterize the meaningful complexity in $X$ as the length of its shortest meaningful description:
\begin{equation}
    \min_{S \in \mathcal{S}} \{l(S) | S \vdash X\}\,,
\end{equation}
where $l(\cdot)$ denotes the length of a description. However, this naive approach returns us to the problem of measuring random noise as highly complex, because if we are restricted only to meaningful descriptions, then we would need a very long one to completely describe a piece of noise. Instead, the approach taken both by our work, and by sophistication and effective complexity on the theoretical side, is to make use of the non-meaningful portion, not to count directly towards the complexity score, but in selecting the shortest description. The amount of meaningful complexity in $X$ is measured as
\begin{equation}
    l(D^*), \text{  where  } (D^*,E^*) = \min_{(D,E) \in \mathcal{S} \times \mathcal{R}} \{l(D)+l(E) | D,E \vdash X\}\,.
\end{equation}
This leads to random noise getting a high value of $l(R)$, but a low value of $l(S)$, so even though its overall description length, $l(S) + l(R)$, might be high, the resulting complexity score is low. In our case in particular, as shown by Theorem \ref{main-theorem}, the total description length tends to be minimized by having a single cluster, which means that the meaningful description is essentially of zero length and the entirety of the data is specified directly as outliers via \eqref{eq:indexing-epsilon-hypercube}.

There are important differences between our method and these theoretical works as well: in order to capture local spatial information, we measure the entropy of cluster labels \emph{within patches}, not of individual points; and we repeat our method recursively at different levels, to capture compositionality, as described in Section~\ref{sec:introduction}. However, to the problem of correctly measuring the complexity of noise, our method uses, on a high level, the same solution as that explored in the theoretical concepts of sophistication and effective complexity.

\subsection{Worked Example} \label{subsec:worked-example}
This section contains a worked example on a randomly chosen image from ImageNet, shown in Figure~\ref{fig:golf-balls}. The steps of our method are enumerated for each of the four levels of the hierarchy. This shows how the final complexity score is obtained. At each level $i$, the model 
\begin{enumerate}
    \item performs MDL clustering on the set of array elements $B_{i-1}$, and assigns each a cluster label, to form $A_i$ (initially, $B_0$ is an image array of pixels, and then $A_1$ contains a cluster label for each pixel in $B_0$)
    \item forms $B_i$ out of patch signatures of multisets of labels in each patch of $A_i$ 
\end{enumerate}

\begin{figure*}
    \centering
    \begin{tabular}{@{}c@{}c@{}}
    \includegraphics[width=0.45\textwidth]{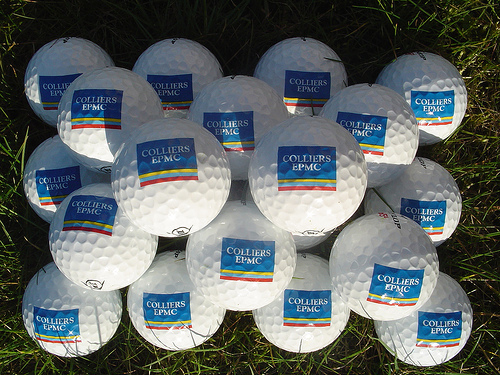} & 
    \hspace{1em}
    \includegraphics[width=0.454\textwidth]{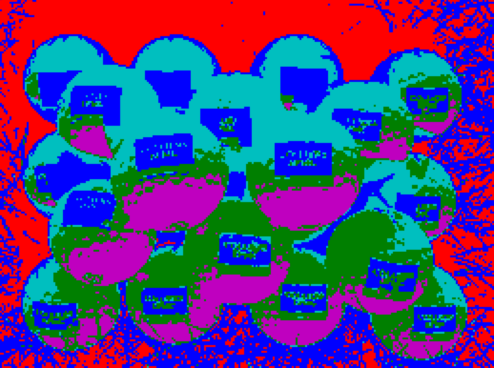}
    \end{tabular}
    \caption{ \small Left: Example of a relatively high-resolution real-world image from ImageNet. ID: n03445777\_10762.
    Right: The matrix $A$ formed by MDL clustering each point in the input image, i.e., by applying the function MDL\_Cluster from Algorithm \ref{alg:method}; different cluster indices are shown in different colours.}
    \label{fig:golf-balls}
\end{figure*}
\noindent  \textbf{Layer 1}: 50246 points to cluster (pixels)  \\
Number of components found by MDL, as per \eqref{eq:k-star}: 7 \\
Assign each pixel a label from $0, \dots, 6$, and form patch signatures as multisets of labels inside all $4 \times 4$ patches, which gives 48450 patches, of 1411 different unique values. \\
Entropy of resulting categorical distribution of patch signatures: \textbf{7.995} \\

\noindent  \textbf{Layer 2}: 48450 points to cluster (pixels)  \\
Number of components found by MDL, as per \eqref{eq:k-star}: 8 \\
Assign each point a label from $0, \dots, 7$, and form patch signatures as multisets of labels inside all $8 \times 8$ patches, which gives 44954 patches, of 3677 different unique values. \\
Entropy of resulting categorical distribution of patch signatures: \textbf{10.194} \\

\noindent  \textbf{Layer 3}: 44954 points to cluster (pixels)  \\
Number of components found by MDL, as per \eqref{eq:k-star}: 8 \\
Assign each point a label from $0, \dots, 7$, and form patch signatures as multisets of labels inside all $16 \times 16$ patches, which gives 38346 patches, of 7341 different unique values. \\
Entropy of resulting categorical distribution of patch signatures: \textbf{12.772} \\

\noindent  \textbf{Layer 4}: 38346 points to cluster (pixels)  \\
Number of components found by MDL, as per \eqref{eq:k-star}: 7 \\
Assign each point a label from $0, \dots, 6$, and form patch signatures as multisets of labels inside all size $32 \times 32$ patches, which gives 26666 patches, of 5666 different unique values. \\
Entropy of resulting categorical distribution of patch signatures: \textbf{12.353} \\
\\
Total complexity: $7.995 + 10.194 + 12.772 + 12.753 = \boldsymbol{43.314}$
\\

\section{Experimental Evaluation} \label{sec:experimental-evaluation}
It is difficult to assess the performance of an image complexity measure empirically. Some works gather human subjective judgements on a particular distribution of images (e.g., European renaissance paintings) and report accuracy/correlation, often also training a supervised model on these human judgements~\citep{machado2015computerized,nagle2020predicting}. Aside from the practical difficulties of running these psychological studies, evaluating a model on a single distribution does not give a rounded indication of its accuracy, it is unclear how such models will perform when presented with a more diverse set of images. Additionally, collecting human judgments of complexity in this way may not be reliable: they are influenced by the presentation of the image as well as cognitive factors such as visual working memory \citep{sherman2013visual}, and show high inter-subject variability \citep{madrid2019human}. There is also EEG evidence suggesting that humans use different cognitive processes to judge an image's complexity depending on its degree of naturalness/familiarity \citep{nicolae2020preparatory}. We instead evaluate this method with a number of different experiments that, together, show that it assigns complexity scores in a coherent and consistent way, and that it accords with our intuitive understanding of complexity. 

Firstly, we present the scores produced by our method for a diverse set of images of different types, taken from different datasets, both public and synthetic datasets that we create, and compare these scores to those produced by existing complexity metrics. Comparing sets/types of images, rather than individual images, has the advantage of reducing subjectivity. One can say with reasonable objectivity that ImageNet images are more complex than MNIST images, whereas  there is more subjectivity in trying to compare the complexity of two different Renaissance paintings, or even two different ImageNet images. The scores produced by our method match our intuitive notion of complexity on this diverse set of images much more closely than do the scores of existing complexity metrics.

Then, after presenting ablation studies, we investigate the distribution of complexity across different levels of the hierarchy, and show that these agree with the different scales of complexity in the different types of images, e.g., fine-detailed repetitive textures receive high scores on the low levels of the hierarchy but lower scores on the higher levels, compared to globally structured images such as natural scenes from ImageNet.

Next, we show the effect of adding Gaussian noise and of lowering the resolution of images. A small amount of noise or reduction in resolution does not change the content of the image and so should not have a significant effect on the complexity score. For larger reductions in image quality, we would expect a gradual decline in complexity as the information in the image becomes increasingly obscured. This is exactly the case for our method. Its scores are largely unchanged by small quality degradations (addition of noise or reduction in resolution), and then show a steady decline with increasing degradation. As our method so effectively assigns low complexity to white noise images, it is particularly notable that it remains robust to a small/moderate amount of Gaussian noise.

Finally, we present the scores produced by our method on a fractal image, as the fractal dimension is varied. Again, the results are in line with our intuition about the type of complexity expressed by fractal dimension: higher fractal dimensions get a higher complexity score, but this is largely concentrated on the lower, more local levels.

\subsection{Datasets}
We present the average score of our method on seven different sets of images, four popular image datasets and three synthetic datasets that we created: 
\begin{enumerate}
\item \textbf{ImageNet} is a dataset with high complexity, depicting real-world objects in context. 
\item \textbf{CIFAR} also shows real-world objects in context but of a much lower resolution, $32 \times 32$ vs.~approximately $224 \times 224$ for ImageNet. 
\item \textbf{MNIST} depicts low-resolution greyscale digits. Its images are simple in that they can be represented exactly with a small number of bits, but still have meaningful semantic content. 
\item \textbf{DTD2} is a dataset that we created by manually searching through the Describable Textures Dataset \parencite{cimpoi2014describing} for all images of fine-detailed repeating textures.
\item \textbf{Stripes} is a synthetic dataset that we created of greyscale images of stripes of varying thickness and orientation. The thickness of the lines, in pixels, is sampled uniformly at random from $[3,10]$, and the slope of the lines is sampled uniformly at random from $[-0.5, -1.5]$. It is sufficient to consider negative slopes only as our method, and all methods that we compare to, are invariant to reflections, so the striped images with slope in $[0.5, 1.5]$ would receive identical scores to those in $[-0.5, -1.5]$. Note that our method is not necessarily invariant to rotations, because it is based on square, axis-aligned patches of pixels. The same is true of the fractal dimension computed with the Minkowski-Bouligand dimension (i.e., the fractal dimension), as it uses a box-counting method. An example of an image from Stripes images is shown in Figure \ref{fig:HMDLC-synthetic-dsets-images}.
\\
\item \textbf{Halves} is a synthetic dataset that we created  of greyscale images of half-black and half-white. These images have one half entirely black and the other entirely white, with the dividing line at various angles. As with Stripes, the slope of this dividing line is sampled uniformly at random from $[-0.5, -1.5]$. An example of an image from Halves is shown in Figure \ref{fig:HMDLC-synthetic-dsets-images}.
\\
\item \textbf{Rand} is a synthetic dataset that we created of white noise images, i.e., images with independent random pixel values. Their values are sampled uniformly at random from $[0,1]$, independently for each location and each of three colour channels. Figure \ref{fig:HMDLC-synthetic-dsets-images} shows an example image.
\\
\end{enumerate}
\begin{figure}\[ht\]
    \centering
    \begin{tabular}{@{}c@{}c@{}c@{}}
    \includegraphics[width=0.33\textwidth]{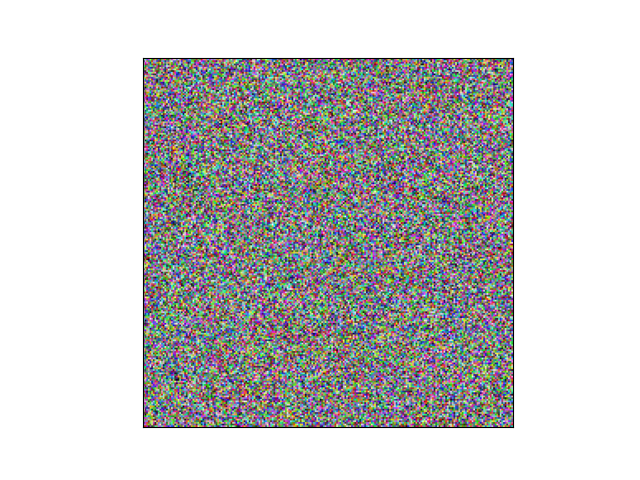} & 
    \includegraphics[width=0.33\textwidth]{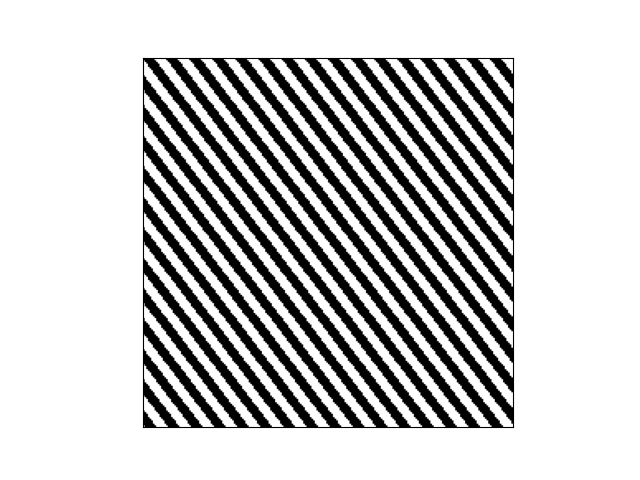} & 
    \includegraphics[width=0.33\textwidth]{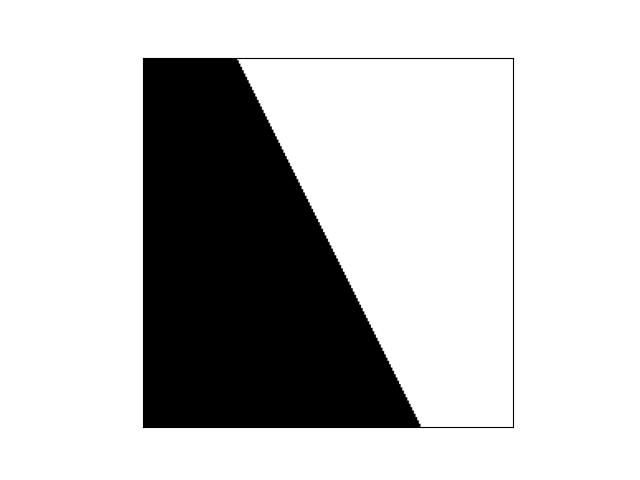}  
    \end{tabular}
    \caption{ \small Examples of images from the synthetic datasets we create. Left: Rand dataset; Middle: Stripes dataset; Right: Halves dataset.}
    \label{fig:HMDLC-synthetic-dsets-images}
\end{figure}
For DTD2, we find $341$ suitable images. For all other datasets, we use $1500$ randomly sampled images and report the average for each image type.  
All images are resized to $224 \times 224$. The GMMs used for clustering are initialized with k-means, use diagonal covariance matrices, have tolerance $1e-3$, and are capped at $100$ iterations.

\subsection{Comparison with Existing Methods} \label{subsec:main-results}
Table~\ref{tab:main-results} compares our method to seven others: `khan2021' \citep{khan2022leveraging}, `machado2015' \citep{machado2015computerized}, and `redies2012' \citep{redies2012phog} are as described in Section \ref{sec:related-work}; `entropy' converts the image to greyscale, discretizes the values into 256 bins, and then computes the Shannon entropy of the bin counts; `fractal dim.' converts the image to greyscale, then binarizes it to $0$ or $1$, and computes the fractal dimension of the resulting shape using the box-counting method; `jpg-ratio' measures the ratio of the JPEG-compressed file size to that of the original; and `GLCM' computes the average entropy of the grey-level co-occurrence matrix, at offsets $1$, $4$, $8$, $16$, and $32$ (see \citet{sebastian2012gray} for an account of GLCM in image complexity). All methods are normalized so their maximum score is 1.

The most striking result is that our method assigns zero complexity to white-noise images, while every other method assigns them high complexity, with many assigning maximum complexity. White noise images are not at all meaningful or interesting to humans, and it is a significant finding that our method is the first to reflect this. It suggests that, while existing methods are based only on the variation across the image, our method is able to measure the degree of \emph{meaningful} variation, i.e., it is able to distinguish signal from noise.

The only two existing methods not to measure white noise as maximally complex are `machado2015' and `redies2012', though they still give it a high score. Instead, they give their max score to Stripes. This is also undesirable, because the simple repeating black and white stripes are not intuitively complex or meaningful either. These methods are both based on gradients (see Section \ref{sec:related-work}), and the stripes produce a sharp gradient at every transition from black to white, which is likely the reason for these high scores. 
Stripes is also given a high score by the fractal dimension and JPEG-ratio methods, both assigning it only slightly less than white noise and significantly more than any other dataset, including ImageNet. 
The method of \citep{khan2022leveraging} (denoted `khan2021') is difficult to interpret at all, because it assigns such a high score to the white noise that, after normalizing, all other datasets end up close to zero, with three being equal to zero. Recall that this method takes the median of the Fourier transform coefficients, so equals zero if over half of the coefficients are zero. 
Perhaps surprisingly, the relatively simple methods of entropy and GLCM entropy do a reasonable job of distinguishing real-world images from synthetic images and MNIST, compared to the more bespoke methods. However, they cannot detect a significant difference between ImageNet, CIFAR, and DTD, assigning all three very similar scores. In contrast, our method agrees much more closely with the intuitive notion of complexity: it assigns the highest complexity to ImageNet; it puts CIFAR ahead of DTD2 even though the latter is of higher resolution and has a complex texture, which shows that it recognizes CIFAR to have more semantically meaningful content; and it assigns MNIST a reasonably high complexity, despite it being the smallest in terms of file size, again showing that it can recognize global structure. Even aside from the white noise, no method but ours correctly places the remaining six datasets in order of complexity (left-to-right, as they appear in Table~\ref{tab:main-results}). This highlights the superior ability of our method to capture meaningful complexity across a variety of image types. 

\begin{table}
\caption{Comparison of our method with existing methods. The figures for each dataset are the mean across all images from that dataset, with std.~dev. from batches of 25 in parentheses. All methods are normalized, so the maximum score that they assign is 1. Ours is the only method that does not assign white noise images high complexity, and gives the most reasonable results on all other datasets.}
\label{tab:main-results}
\centering
\resizebox{\textwidth}{!}{\begin{tabular}{*{8}{l}}
\toprule
 & \multicolumn{7}{c}{\textbf{Dataset}}\\
\cmidrule(r){2-8}\\
& \textbf{ImageNet} & \textbf{CIFAR} & \textbf{DTD2} & \textbf{MNIST} & \textbf{Stripes} & \textbf{Halves} & \textbf{white-noise}\\
\textbf{ours} & 0.80 (.10) & 0.74 (.06) & 0.62 (.29) & 0.50 (.08) & 0.36 (.11) & 0.26 (.01) & 0.00 (.00)\\
\textbf{khan2021} & 0.09 (.05) & 0.01 (.01) & 0.07 (.06) & 0.00 (.00) & 0.00 (.00) & 0.00 (.00) & 0.99 (.00)\\
\textbf{machado2015} & 0.23 (.08) & 0.15 (.02) & 0.38 (.08) & 0.21 (.01) & 0.53 (.02) & 0.06 (.00) & 0.87 (.00)\\
\textbf{redies2012} & 0.13 (.05) & 0.04 (.01) & 0.21 (.11) & 0.00 (.00) & 0.66 (.34) & 0.01 (.00) & 0.59 (.00)\\
\textbf{entropy} & 0.89 (.10) & 0.89 (.07) & 0.83 (.13) & 0.30 (.06) & 0.13 (.00) & 0.13 (.00) & 0.96 (.00)\\
\textbf{fractal dim.} & 0.74 (.09) & 0.61 (.08) & 0.86 (.16) & 0.45 (.06) & 0.98 (.02) & 0.44 (.02) & 1.00 (.00)\\
\textbf{jpg-ratio} & 0.22 (.08) & 0.09 (.0) & 0.29 (.09) & 0.06 (.01) & 0.57 (.01) & 0.06 (.00) & 0.57 (.00)\\
\textbf{GLCM} & 0.84 (.11) & 0.80 (.08) & 0.83 (.14) & 0.27 (.05) & 0.11 (.02) & 0.08 (.00) & 0.98 (.00)\\
\bottomrule
\end{tabular}}\vspace*{-2ex}
\end{table}

\subsection{Ablation Studies} \label{subsec:ablation-studies}

\begin{table}
\caption{Effect of removing two main components of our method. In `no mdl', clustering is performed without MDL, instead simply fixing the number of clusters to $5$ for all images and all scales. In `no patch', we compute the entropy of the clusters themselves rather than of the patch signatures.}
\label{tab:ablation-results}
\centering
\resizebox{\textwidth}{!}{\begin{tabular}{*{8}{l}}
\toprule
 & \multicolumn{7}{c}{\textbf{Dataset}}\\
\cmidrule(r){2-8}\\
& \textbf{ImageNet} & \textbf{CIFAR} & \textbf{DTD2} & \textbf{MNIST} & \textbf{Stripes} & \textbf{Halves} & \textbf{white-noise}\\
\textbf{main} & 0.80 (.10) & 0.74 (.06) & 0.62 (.29) & 0.50 (.08) & 0.36 (.11) & 0.26 (.01) & 0.00 (.00)\\
\textbf{no mdl} & 0.73 (.09) & 0.66 (.06) & 0.90 (.11) & 0.40 (.07) & 0.35 (.13) & 0.27 (.01) & 0.98 (.00)\\
\textbf{no patch} & 0.92 (.09) & 94 (.04) & 0.62 (.28) & 0.61 (.1) & 0.74 (.09) & 0.50 (.01) & 0.00 (.00)\\
\bottomrule
\end{tabular}}
\end{table}

Table~\ref{tab:ablation-results} shows the effect of removing two key components of our method. In `no mdl', rather than selecting the number of clusters $K$ using the minimum description length principle, we fix $K=5$ for all images. This results in the same problem that existing methods suffer from: white noise is mistaken for high complexity and receives the maximum score. Also, `no mdl' scores DTD2 too highly, showing that the method is not responding to global structure. In `no patch', we take the entropy not of patch signatures, but of individual points in the array, i.e., of $A$ rather than $B$ in the terminology of Section~\ref{subsec:patch-entropy}. (Patch signatures are still used for the iteration step.) This setting still performs reasonably well, but it gives too high a score to Stripes and a higher score to CIFAR than to ImageNet.

\subsection{Complexity at Different Scales} \label{subsec:complexity-at-scales}
The results from Section \ref{subsec:main-results} suggest that, unlike existing methods, which focus only on detailed textures, ours is able to recognize complexity at a global level. Figure~\ref{fig:complexity-by-scale} provides further support for this claim by showing the breakdown of our complexity measure at the four different scales (that is, four different levels of the hierarchy; see Section \ref{subsec:patch-entropy}). Smaller scales respond to local complexity, and as the process is iterated to larger scales, global structure can be detected.

\begin{figure*}[t]
    \centering
    \begin{tabular}{@{}c@{}c@{}}
    \includegraphics[width=0.53\textwidth]{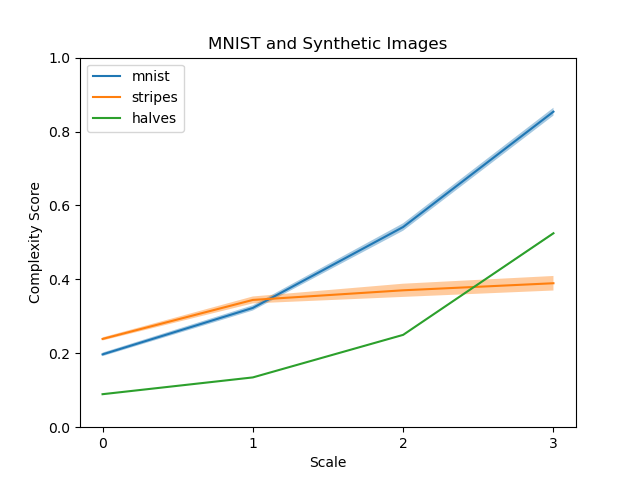} & 
    \includegraphics[width=0.53\textwidth]{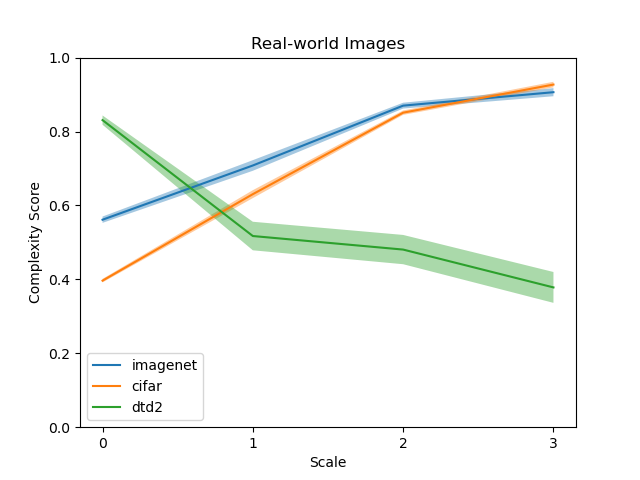} \\
    \end{tabular}
    \caption{ \small  \small Our complexity measure for different scales. The $x$-axis depicts patch size, on a log scale. Plots show mean score for all images of that type. Shaded regions are std dev from batches of 25 images.}
    \label{fig:complexity-by-scale}
\end{figure*}

The first plot shows MNIST and the synthetic images. While MNIST has a similar local complexity score to Stripes, it has a much higher global complexity score, indicating that the more meaningful global structure in MNIST images can be detected. Halves, which is almost uniform locally but shows some variation globally, is given a very low local complexity but a small amount of global complexity. The second plot compares real-world images. CIFAR has the lowest local complexity, because it is low resolution, because it was resized from $32 \times 32$, so neighbouring pixels are all similar, but this does not affect its global complexity, which is as high as that of ImageNet. DTD2, on the other hand, has the highest local complexity,  because it depicts detailed textures, but the lowest global complexity, because the textures are uniform across different regions of the images. 


\subsection{Effect of Adding Gaussian Noise}
As our method so consistently assigns zero complexity to white noise, one may wonder whether it just searches for randomness in the image, and assigns zero if it finds any. To check this, we progressively add Gaussian noise to the three real-world datasets. The results are shown in Figure~\ref{fig:add-gaussian-noise}. Noise is sampled independently from a standard normal distribution for each pixel, and a fraction of this noise is added to the image. Up until 10\%, the scores are largely unchanged (DTD drops slightly), and then the scores for all three datasets steadily decrease with further noise. If the method was simply assigning low complexity in response to any randomness in the image, then we would see a sharp decline as soon as a small amount of noise is added. The results suggest that the method is instead responding to the amount of meaningful content in the image. A gradual decline in complexity is precisely what we would expect as the image quality deteriorates.

\begin{figure}
    \centering
    \includegraphics[width=.75\textwidth]{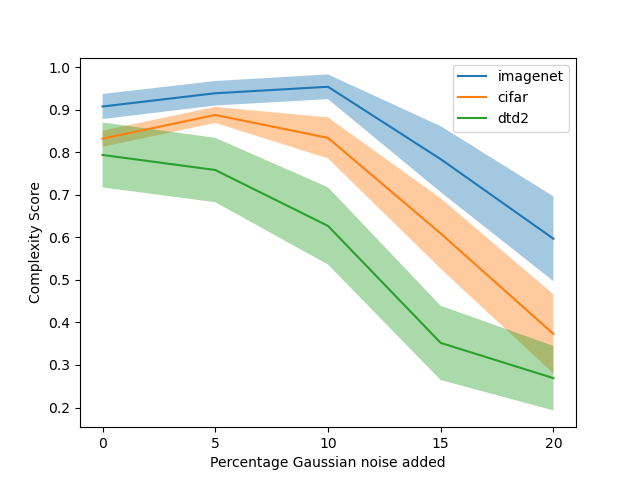}
    \caption{ \small  \small Our complexity measure with different amounts of Gaussian noise added. Shaded regions are std from batches of 25 images. That is, we randomly sample 25 images and compute the mean complexity score, then repeat this for a total of 300 images and report the (unbiased) sample std dev.}
    \label{fig:add-gaussian-noise}
\end{figure}

\subsection{Effect of Changing Resolution}
To investigate how much our method is affected by the resolution of the input image, we apply it to a downsampled ImageNet. We randomly select 300 of the 1500 ImageNet images used for our main experiment and convert them to resolution $32 \times 32$. Table~\ref{tab:downsample-results} shows the results of our method on these downsampled images and compares to the full-sized ImageNet images, which are roughly $256 \times 256$. There is a slight drop on the lower levels of the hierarchy, which corresponds to the greater uniformity at the local scale in the blurry, low-resolution images. The scores at the higher levels are essentially identical, and overall the scores are almost the same for the downsampled images as for the full-resolution images. This shows our method to be robust to changes in resolution, responding more to the contents of the image than to the resolution it is depicted at.

\begin{table}
\caption{ \small  \small Comparison, on the scores produced by our method, of downsampling ImageNet to $32 \times 32$. Taken from 300 randomly sampled images of the 1500 used for the main results in Table~\ref{tab:main-results}.}
\label{tab:downsample-results}
\centering
\begin{tabular}{*{6}{l}}
\toprule
& Level 1 & Level 2 & Level 3 & Level 4 & Total \\
\textbf{full resolution} & 7.70 (1.75) & 9.71 (1.99) & 11.93 (1.75) & 12.43 (1.98) & \textbf{41.77} (5.78) \\
\textbf{low resolution} & 5.60 (0.73) & 9.07 (1.43) & 11.92 (1.14) & 12.72 (0.58) & \textbf{39.03} (3.40) \\
\bottomrule
\end{tabular}
\end{table}

\subsection{Scores for Varying Fractal Dimension}
Fractal dimension can roughly be defined as the detail in a shape or curve expressed as an exponent of its scale. 
(See \cite{edgar2007measure} for discussion of different options for a precise definition.) In this section, we test the scores produced by our complexity metric on images of varying fractal dimension, from the dataset ``Color Fractal Images with Independent RGB Color Components'' \citep{ivanovici2010fractal}. This is a small dataset of nine high-resolution colour images, which are essentially the same except that they differ in fractal dimension. The images are generated using the midpoint displacement algorithm, which iteratively increases the fractal dimension of a piecewise-linear curve (i.e., a joined sequence of straight line segments), by slightly moving the midpoint of each piece. The dataset begins with a straight line and iterates until the fractal dimension is a certain value. The values for the different images range from 1.1 to 1.9 in increments of 0.1. This is repeated independently for each of the three colour channels. The resulting images are shown in Figure~\ref{fig:CFIICC-images}.

\begin{figure}
    \centering
    \includegraphics[width=0.9\textwidth]{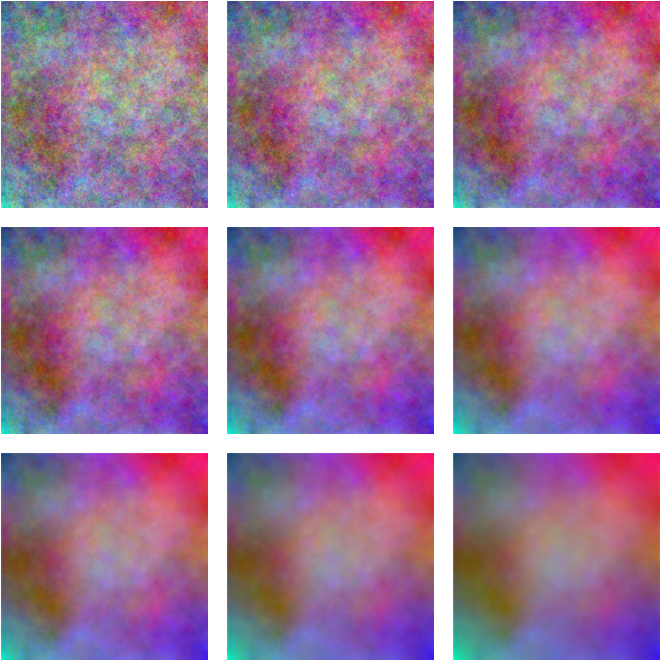}
    \caption{ \small  \small The nine images from the ``Color Fractal Images with Independent RGB Color Components'' dataset. Top-left images is 1.9, decreases in increments of 0.1 to bottom-right.}
    \label{fig:CFIICC-images}
\end{figure}

It is generally thought that a higher fractal dimension indicates greater complexity, and so it is interesting to see whether our method is able to reflect this. Table~\ref{tab:fract-results} shows the scores produced by our method for each of the nine images in ``Color Fractal Images with Independent RGB Color Components'', averaged over five runs, with the clustering at each level using 10 different GMM initializations and keeping the one with the highest data likelihood. The table shows the total complexity score, and the complexity score for each level. Looking at the total scores, there is a clear trend of increasing complexity scores for increasing fractal dimension, showing that the method can detect the sort of complexity expressed in fractal dimension. Looking at the breakdown of this total across the four levels of the hierarchy, we see that the effect of increased fractal dimension is greater for lower levels. Level 1 increases from 3.79 for fractal dimension 1.1 to 10.75 for fractal dimension 1.9, whereas Level 4 shows no systematic increase at all. The same information is shown graphically in Figure~\ref{fig:fract-results-as-plot}. This reflects the similarity of the images in Figure~\ref{fig:CFIICC-images} at a more global level, e.g., they all have a patch of pink/red in the top-right corner and a patch of purple/blue in the bottom-right corner. It is within each patch, that is at a smaller scale, that the images differ. 

\begin{table}
\caption{ \small Scores for increasing fractal dimension. For each fractal dimension, we run our method five times on the single image of that fractal dimension, and compute the (unbiased) sample standard deviation. } \label{tab:fract-results}
    \centering
\resizebox{0.95\textwidth}{!}{
\begin{tabular}{llllll}
\toprule
{} &         Total &        Level 1 &        Level 2 &        Level 3 &        Level 4 \\
\midrule
fract-dim 1.1 &  32.96 (0.15) &   3.79 (0.21) &   6.21 (0.27) &  10.42 (0.46) &  12.54 (0.16) \\
fract-dim 1.2 &  34.39 (0.67) &   4.02 (0.23) &   6.73 (0.33) &  11.22 (0.16) &  12.41 (0.24) \\
fract-dim 1.3 &  34.95 (0.19) &   4.24 (0.02) &   6.95 (0.21) &  11.06 (0.05) &  12.70 (0.00) \\
fract-dim 1.4 &  36.02 (0.46) &   4.60 (0.06) &   7.60 (0.55) &  11.58 (0.39) &  12.24 (0.31) \\
fract-dim 1.5 &  37.67 (0.48) &   5.31 (0.05) &   8.34 (0.39) &  11.63 (0.15) &  12.39 (0.24) \\
fract-dim 1.6 &  39.64 (0.37) &   6.27 (0.18) &   9.36 (0.27) &  11.56 (0.34) &  12.44 (0.14) \\
fract-dim 1.7 &  41.45 (0.30) &   7.35 (0.04) &   9.61 (0.18) &  12.10 (0.27) &  12.38 (0.14) \\
fract-dim 1.8 &  44.94 (0.20) &   8.96 (0.02) &  10.81 (0.15) &  12.59 (0.13) &  12.58 (0.05) \\
fract-dim 1.9 &  47.75 (0.14) &  10.75 (0.07) &  11.84 (0.25) &  12.75 (0.25) &  12.42 (0.29) \\
\bottomrule
\end{tabular}
}
\end{table}
\begin{figure}[h]
    \centering
    \includegraphics[width=0.7\textwidth]{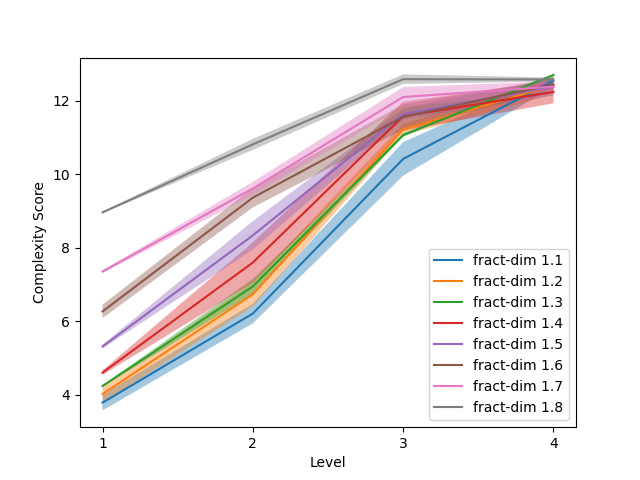}
    \caption{ \small  Trend of increasing complexity score for increasing fractal dimension, broken down by level of the hierarchy. Shaded regions are (unbiased) sample std dev, as in Table \ref{tab:fract-results}.}
    \label{fig:fract-results-as-plot}
\end{figure}

This shows that our method is not only able to reliably detect an increase in fractal dimension by assigning a higher complexity score, it also distributes the increase with fractal dimension over the four levels of the hierarchy in the correct way. There is a significant increase at the most local level and progressively smaller increases at higher levels.

\section{Discussion} \label{sec:conclusion}
\subsection{Limitations and Future Work}
One drawback of the current version of our method is the time complexity. Most existing image complexity metrics run in $<0.1s$ per image, whereas ours takes between $2s$ and $8s$ on average (the simple datasets are faster, ImageNet and CIFAR are the slowest). This can be roughly halved by reducing the range of $K$ explored from 8 to 5, with essentially no change in results. The run time is mostly due to the clustering step on the relatively large number of image patches. One future extension that could significantly improve runtime is, rather than considering all overlapping patches, to determine or approximate an optimal partition of the image into \emph{non-overlapping} patches. This could draw on work in visual tiling \citep{xiao2017optile}. Another future extension is to apply a similar method to other data domains, such as videos, audio, or text.

\subsection{Conclusion}
This paper presented a method for measuring image complexity. 
This task is inspired by the fact that humans cannot only explicitly recognize patterns in data, but can also detect whether the data contain a complex pattern at all. Our method can assign a complexity score to data, specifically to images, which quantifies how complex a pattern or structure they contain. Unlike existing ways of quantifying complexity, it is able to capture the amount of meaningful complexity, and does not judge random noise to be complex.   
 
It uses clustering to analyse an image as being built out of a hierarchy of patches, with each patch composed of the cluster indices of its sub-patches. Clustering is performed with the minimum description length principle to distinguish signal from noise. We gave a detailed derivation of our method, and then presented an experimental evaluation showing that it performs better than existing measures of image complexity. Most strikingly, it assigns a very low score to white noise, in contrast to existing methods, which all measure white noise as highly complex. This result is also supported theoretically with a proof that white noise contains only one cluster, as judged by MDL, which immediately implies that our method assigns it very low complexity. We then presented ablation studies 
and a further set of experiments showing that it can accurately capture complexity at different scales, that it is robust to small/moderate degradations in quality, either from the addition of Gaussian noise or from a reduction in resolution, and that it can accurately reflect increasing fractal dimension in fractal images.

\section*{Acknowledgements}
This work was supported by the Alan Turing Institute under the UK EPSRC grant EP/N510129/1 and by the AXA Research Fund.

\printbibliography

@article{forsythe2011predicting,
  title={Predicting beauty: {F}ractal dimension and visual complexity in art},
  author={Forsythe, Alex and Nadal, Marcos and Sheehy, Noel and Cela-Conde, Camilo J. and Sawey, Martin},
  journal={Br. J. of Psychology},
  volume={102},
  number={1},
  pages={49--70},
  year={2011},
  publisher={Wiley Online Library}
}

@article{marin2013examining,
  title={Examining complexity across domains: {R}elating subjective and objective measures of affective environmental scenes, paintings and music},
  author={Marin, Manuela M. and Leder, Helmut},
  journal={PloS One},
  volume={8},
  number={8},
  pages={e72412},
  year={2013},
  publisher={Public Library of Science San Francisco, USA}
}

@article{machado2015computerized,
  title={Computerized measures of visual complexity},
  author={Machado, Penousal and Romero, Juan and Nadal, Marcos and Santos, Antonino and Correia, Jo{\~a}o and Carballal, Adri{\'a}n},
  journal={Acta Psychologica},
  volume={160},
  pages={43--57},
  year={2015},
  publisher={Elsevier}
}

@article{lam2002evaluation,
  title={An evaluation of fractal methods for characterizing image complexity},
  author={Lam, Nina Siu-Ngan and Qiu, Hong-lie and Quattrochi, Dale A. and Emerson, Charles W.},
  journal={Cartography and Geographic Information Science},
  volume={29},
  number={1},
  pages={25--35},
  year={2002},
  publisher={Taylor \& Francis}
}

@article{carballal2020comparison,
  title={Comparison of outlier-tolerant models for measuring visual complexity},
  author={Carballal, Adrian and Fernandez-Lozano, Carlos and Rodriguez-Fernandez, Nereida and Santos, Iria and Romero, Juan},
  journal={Entropy},
  volume={22},
  number={4},
  pages={488},
  year={2020},
  publisher={Multidisciplinary Digital Publishing Institute}
}

@article{khan2022leveraging,
  title={Leveraging image complexity in macro-level neural network design for medical image segmentation},
  author={Khan, Tariq M. and Naqvi, Syed S. and Meijering, Erik},
  journal={Scientific Reports},
  volume={12},
  number={1},
  pages={22286},
  year={2022},
  publisher={Nature Publishing Group UK London}
}

@article{cronin2006imitation,
  title={The imitation game --- {A} computational chemical approach to recognizing life},
  author={Cronin, Leroy and Krasnogor, Natalio and Davis, Benjamin G. and Alexander, Cameron and Robertson, Neil and Steinke, Joachim H. G. and Schroeder, Sven L. M. and Khlobystov, Andrei N. and Cooper, Geoff and Gardner, Paul M. and others},
  journal={Nat. Biotechnology},
  volume={24},
  number={10},
  pages={1203--1206},
  year={2006},
  publisher={Nature Publishing Group}
}

@article{marshall2021identifying,
  title={Identifying molecules as biosignatures with assembly theory and mass spectrometry},
  author={Marshall, Stuart M. and Mathis, Cole and Carrick, Emma and Keenan, Graham and Cooper, Geoffrey J. T. and Graham, Heather and Craven, Matthew and Gromski, Piotr S. and Moore, Douglas G. and Walker, Sara and others},
  journal={Nat. Communications},
  volume={12},
  number={1},
  pages={1--9},
  year={2021},
  publisher={Nature Publishing Group}
}

@article{schwieterman2018exoplanet,
  title={Exoplanet biosignatures: {A} review of remotely detectable signs of life},
  author={Schwieterman, Edward W. and Kiang, Nancy Y. and Parenteau, Mary N. and Harman, Chester E. and DasSarma, Shiladitya and Fisher, Theresa M. and Arney, Giada N. and Hartnett, Hilairy E. and Reinhard, Christopher T. and Olson, Stephanie L. and others},
  journal={Astrobiology},
  volume={18},
  number={6},
  pages={663--708},
  year={2018},
  publisher={Mary Ann Liebert, Inc. 140 Huguenot Street, 3rd Floor New Rochelle, NY 10801 USA}
}

@article{forsythe2008confounds,
  title={Confounds in pictorial sets: The role of complexity and familiarity in basic-level picture processing},
  author={Forsythe, Alex and Mulhern, Gerry and Sawey, Martin},
  journal={Behavior Research Methods},
  volume={40},
  number={1},
  pages={116--129},
  year={2008},
  publisher={Springer}
}

@inproceedings{stickel2010xaos,
  title={The {xaos} metric: {u}nderstanding visual complexity as measure of usability},
  author={Stickel, Christian and Ebner, Martin and Holzinger, Andreas},
  booktitle={Proceedings of the Symposium of the Austrian HCI and Usability Engineering Group},
  pages={278--290},
  year={2010},
  organization={Springer}
}

@book{falconer2004fractal,
  title={Fractal Geometry: Mathematical Foundations and Applications},
  author={Falconer, Kenneth},
  year={2004},
  publisher={John Wiley \& Sons}
}

@inproceedings{redies2012phog,
  title={Phog-derived aesthetic measures applied to color photographs of artworks, natural scenes and objects},
  author={Redies, Christoph and Amirshahi, Seyed Ali and Koch, Michael and Denzler, Joachim},
  booktitle={Proceedings of the European Conference on Computer Vis.},
  pages={522--531},
  year={2012},
  organization={Springer}
}

@inproceedings{cimpoi2014describing,
  title={Describing textures in the wild},
  author={Cimpoi, Mircea and Maji, Subhransu and Kokkinos, Iasonas and Mohamed, Sammy and Vedaldi, Andrea},
  booktitle={Proceedings of the IEEE Conference on Computer Vis. and Pattern Recognit.},
  pages={3606--3613},
  year={2014}
}

@article{sebastian2012gray,
  title={Gray level co-occurrence matrices: generalisation and some new features},
  author={Sebastian V., Bino and Unnikrishnan, A. and Balakrishnan, Kannan},
  journal={arXiv preprint ArXiv:1205.4831},
  year={2012}
}

@article{marshall2019quantifying,
  title={Quantifying the pathways to life using assembly spaces},
  author={Marshall, Stuart M. and Moore, Douglas and Murray, Alastair R. G. and Walker, Sara I. and Cronin, Leroy},
  journal={arXiv preprint ArXiv:1907.04649},
  year={2019}
}

@article{rissanen1983universal,
  title={A universal prior for integers and estimation by minimum description length},
  author={Rissanen, Jorma},
  journal={The Annals of Statistics},
  volume={11},
  number={2},
  pages={416--431},
  year={1983},
  publisher={Institute of Mathematical Statistics}
}

@article{sun2006fractal,
  title={Fractal analysis of remotely sensed images: {A} review of methods and applications},
  author={Sun, W. and Xu, G. and Gong, P. and Liang, S.},
  journal={International J. of Remote Sens.},
  volume={27},
  number={22},
  pages={4963--4990},
  year={2006},
  publisher={Taylor \& Francis}
}

@inproceedings{yang2000analysis,
  title={Analysis of the complexity of remote sensing image and its role on image classification},
  author={Yang, Xiaomei and Zhou, Chenghu},
  booktitle={Proceedings of the IGARSS 2000. IEEE 2000 International Geoscience and Remote Sens. Symposium. Taking the Pulse of the Planet: The Role of Remote Sens. in Managing the Environment. Proc. (Cat. No. 00CH37120)},
  volume={5},
  pages={2179--2181},
  year={2000},
  organization={IEEE}
}

@book{birkhoff1933aesthetic,
  title={Aesthetic Measure},
  author={Birkhoff, George David},
  year={1933},
  publisher={Harvard University Press, Cambridge}
}

@inproceedings{madrid2019human,
  title={Human image complexity analysis using a fuzzy inference system},
  author={Madrid-Herrera, Luis and Chacon-Murguia, Mario I. and Posada-Urrutia, Daniel A. and Ramirez-Quintana, Juan A.},
  booktitle={Proceedings of the 2019 IEEE International Conference on Fuzzy Systems (FUZZ-IEEE)},
  pages={1--6},
  year={2019},
  organization={IEEE}
}

@article{nagle2020predicting,
  title={Predicting human complexity perception of real-world scenes},
  author={Nagle, Fintan and Lavie, Nilli},
  journal={Royal Society Open Science},
  volume={7},
  number={5},
  pages={191487},
  year={2020},
  publisher={The Royal Society}
}

@article{sherman2013visual,
  title={Visual-object working memory affects aesthetic judgments},
  author={Sherman, Aleksandra and Lim, So Yum and Grabowecky, Marcia and Suzuki, Satoru},
  journal={J. of Vis.},
  volume={13},
  number={9},
  pages={1308--1308},
  year={2013},
  publisher={The Association for Research in Vision and Ophthalmology}
}

@article{nicolae2020preparatory,
  title={Preparatory experiments regarding human brain perception and reasoning of image complexity for synthetic color fractal and natural texture images via EEG},
  author={Nicolae, Irina E. and Ivanovici, Mihai},
  journal={Appl. Sciences},
  volume={11},
  number={1},
  pages={164},
  year={2020},
  publisher={MDPI}
}

@article{narayanan2003effects,
  title={Effects of noise on the information content of remote sensing images},
  author={Narayanan, Ram M. and Ponnappan, Sudhir K. and Reichenbach, Stephen E.},
  journal={Geocarto International},
  volume={18},
  number={2},
  pages={15--26},
  year={2003},
  publisher={Taylor \& Francis}
}

@article{chioukh2014noise,
  title={Noise and sensitivity of harmonic radar architecture for remote sensing and detection of vital signs},
  author={Chioukh, Lydia and Boutayeb, Halim and Deslandes, Dominic and Wu, Ke},
  journal={IEEE Transactions on Microw. Theory and Techniques},
  volume={62},
  number={9},
  pages={1847--1855},
  year={2014},
  publisher={IEEE}
}

@article{landgrebe1986noise,
  title={Noise in remote-sensing systems: The effect on classification error},
  author={Landgrebe, David A. and Malaret, Erick},
  journal={IEEE Transactions on Geoscience and Remote Sens.},
  number={2},
  pages={294--300},
  year={1986},
  publisher={IEEE}
}

@article{chang2016remote,
  title={Remote sensing image stripe noise removal: From image decomposition perspective},
  author={Chang, Yi and Yan, Luxin and Wu, Tao and Zhong, Sheng},
  journal={IEEE Transactions on Geoscience and Remote Sens.},
  volume={54},
  number={12},
  pages={7018--7031},
  year={2016},
  publisher={IEEE}
}

@article{rasti2018noise,
  title={Noise reduction in hyperspectral imagery: Overview and application},
  author={Rasti, Behnood and Scheunders, Paul and Ghamisi, Pedram and Licciardi, Giorgio and Chanussot, Jocelyn},
  journal={Remote Sens.},
  volume={10},
  number={3},
  pages={482},
  year={2018},
  publisher={MDPI}
}

@inproceedings{huang2020efficient,
  title={Efficient GAN-based remote sensing image change detection under noise conditions},
  author={Huang, Wenzhun and Zhang, Shanwen and Wang, Harry Haoxiang},
  booktitle={Proceedings of the International Conference on Image Processing and Capsule Networks},
  pages={1--8},
  year={2020},
  organization={Springer}
}

@article{duan2019noise,
  title={Noise-robust hyperspectral image classification via multi-scale total variation},
  author={Duan, Puhong and Kang, Xudong and Li, Shutao and Ghamisi, Pedram},
  journal={IEEE J. of Selected Topics in Appl. Earth Observations and Remote Sens.},
  volume={12},
  number={6},
  pages={1948--1962},
  year={2019},
  publisher={IEEE}
}

@book{edgar2007measure,
  title={Measure, Topology, and Fractal Geometry},
  author={Edgar, Gerald},
  year={2007},
  publisher={Springer Science \& Business Media}
}

@article{ivanovici2010fractal,
  title={Fractal dimension of color fractal images},
  author={Ivanovici, Mihai and Richard, No{\"e}l},
  journal={IEEE Transactions on Image Processing},
  volume={20},
  number={1},
  pages={227--235},
  year={2010},
  publisher={IEEE}
}

@article{hulle2005edgeworth,
  title={Edgeworth approximation of multivariate differential entropy},
  author={Hulle, Marc M. Van},
  journal={Neural Computation},
  volume={17},
  number={9},
  pages={1903--1910},
  year={2005},
  publisher={MIT Press One Rogers Street, Cambridge, MA 02142-1209, USA journals-info~…}
}

@inproceedings{pichler2022differential,
  title={A differential entropy estimator for training neural networks},
  author={Pichler, Georg and Colombo, Pierre Jean A. and Boudiaf, Malik and Koliander, G{\"u}nther and Piantanida, Pablo},
  booktitle={Proceedings of the International Conference on Machine Learning},
  pages={17691--17715},
  year={2022},
  organization={PMLR}
}

@book{thomas2006elements,
  title={Elements of Information Theory},
  author={Thomas, MTCAJ and Joy, A. Thomas},
  year={2006},
  publisher={Wiley-Interscience}
}

@incollection{huff2021neuroanatomy,
  title={Neuroanatomy, visual cortex},
  author={Huff, Trevor and Mahabadi, Navid and Tadi, Prasanna},
  booktitle={StatPearls [Internet]},
  year={2021},
  publisher={StatPearls Publishing}
}

@article{chater2005minimum,
  title={A minimum description length principle for perception},
  author={Chater, Nick},
  journal={Advances in Minim. Description Length: Theory and Applications},
  pages={372--398},
  year={2005},
  publisher={MIT Press Cambridge, MA}
}

@article{feldman2016simplicity,
  title={The simplicity principle in perception and cognition},
  author={Feldman, Jacob},
  journal={Wiley Interdisciplinary Reviews: Cognitive Science},
  volume={7},
  number={5},
  pages={330--340},
  year={2016},
  publisher={Wiley Online Library}
}

@article{yang2008unsupervised,
  title={Unsupervised segmentation of natural images via lossy data compression},
  author={Yang, Allen Y. and Wright, John and Ma, Yi and Sastry, S. Shankar},
  journal={Computer Vis. and Image Underst.},
  volume={110},
  number={2},
  pages={212--225},
  year={2008},
  publisher={Elsevier}
}

@article{sims2016rate,
  title={Rate--distortion theory and human perception},
  author={Sims, Chris R.},
  journal={Cognition},
  volume={152},
  pages={181--198},
  year={2016},
  publisher={Elsevier}
}

@inproceedings{gibson2002visual,
  title={Visual abstraction of wildlife footage using {G}aussian mixture models and the minimum description length criterion},
  author={Gibson, David and Campbell, Neill and Thomas, Barry},
  booktitle={Proceedings of the 2002 International Conference on Pattern Recognit.},
  volume={2},
  pages={814--817},
  year={2002},
  organization={IEEE}
}

@article{davies2002minimum,
  title={A minimum description length approach to statistical shape modeling},
  author={Davies, Rhodri H. and Twining, Carole J. and Cootes, Timothy F. and Waterton, John C. and Taylor, Christopher J.},
  journal={IEEE Transactions on Méd. Imaging},
  volume={21},
  number={5},
  pages={525--537},
  year={2002},
  publisher={IEEE}
}

@article{vitanyi2006meaningful,
  title={Meaningful information},
  author={Vit{\'a}nyi, Paul M.},
  journal={IEEE Transactions on Information Theory},
  volume={52},
  number={10},
  pages={4617--4626},
  year={2006},
  publisher={IEEE}
}

@article{koppel1987complexity,
  title={Complexity, depth, and sophistication},
  author={Koppel, Moshe},
  journal={Complex Systems},
  volume={1},
  number={6},
  pages={1087--1091},
  year={1987}
}

@article{ay2010effective,
  title={Effective complexity and its relation to logical depth},
  author={Ay, Nihat and M{\"u}ller, Markus and Szkola, Arleta},
  journal={IEEE Transactions on Information Theory},
  volume={56},
  number={9},
  pages={4593--4607},
  year={2010},
  publisher={IEEE}
}

@article{gell1996information,
  title={Information measures, effective complexity, and total information},
  author={Gell-Mann, Murray and Lloyd, Seth},
  journal={Complex.},
  volume={2},
  number={1},
  pages={44--52},
  year={1996},
  publisher={Wiley Online Library}
}

@inproceedings{xiao2017optile,
  title={Optile: toward optimal tiling in 360-degree video streaming},
  author={Xiao, Mengbai and Zhou, Chao and Liu, Yao and Chen, Songqing},
  booktitle={Proceedings of the 25th ACM International Conference on Multimedia},
  pages={708--716},
  year={2017}
}

@article{zhang2023fully,
  title={Fully context-aware image inpainting with a learned semantic pyramid},
  author={Zhang, Wendong and Wang, Yunbo and Ni, Bingbing and Yang, Xiaokang},
  journal={Pattern Recognition},
  pages={109741},
  year={2023},
  publisher={Elsevier}
}

@article{peng2011image,
  title={Image segmentation by iterated region merging with localized graph cuts},
  author={Peng, Bo and Zhang, Lei and Zhang, David and Yang, Jian},
  journal={Pattern Recognition},
  volume={44},
  number={10-11},
  pages={2527--2538},
  year={2011},
  publisher={Elsevier}
}

@article{ilea2011image,
  title={Image segmentation based on the integration of colour--texture descriptors—A review},
  author={Ilea, Dana E. and Whelan, Paul F.},
  journal={Pattern Recognition},
  volume={44},
  number={10-11},
  pages={2479--2501},
  year={2011},
  publisher={Elsevier}
}

@article{passat2011interactive,
  title={Interactive segmentation based on component-trees},
  author={Passat, Nicolas and Naegel, Beno{\i}t and Rousseau, Fran{\c{c}}ois and Koob, M{\'e}riam and Dietemann, Jean-Louis},
  journal={Pattern Recognition},
  volume={44},
  number={10-11},
  pages={2539--2554},
  year={2011},
  publisher={Elsevier}
}

@article{geng2011face,
  title={Face recognition based on the multi-scale local image structures},
  author={Geng, Cong and Jiang, Xudong},
  journal={Pattern Recognition},
  volume={44},
  number={10-11},
  pages={2565--2575},
  year={2011},
  publisher={Elsevier}
}

@inproceedings{caron2018deep,
  title={Deep clustering for unsupervised learning of visual features},
  author={Caron, Mathilde and Bojanowski, Piotr and Joulin, Armand and Douze, Matthijs},
  booktitle={Proceedings of the European Conference on Computer Vision (ECCV)},
  pages={132--149},
  year={2018}
}

@inproceedings{mahon2021selective,
  title={Selective pseudo-label clustering},
  author={Mahon, Louis and Lukasiewicz, Thomas},
  booktitle={KI 2021: Advances in Artificial Intelligence: 44th German Conference on AI, Virtual Event, September 27--October 1, 2021, Proceedings 44},
  pages={158--178},
  year={2021},
  organization={Springer}
}

@article{fang2023robust,
  title={Robust image clustering via context-aware contrastive graph learning},
  author={Fang, Uno and Li, Jianxin and Lu, Xuequan and Mian, Ajmal and Gu, Zhaoquan},
  journal={Pattern Recognition},
  volume={138},
  pages={109340},
  year={2023},
  publisher={Elsevier}
}

@article{mcinnes2017hdbscan,
  title={{HDBSCAN: H}ierarchical density based clustering},
  author={McInnes, Leland and Healy, John and Astels, Steve},
  journal={J. Open Source Softw.},
  volume={2},
  number={11},
  pages={205},
  year={2017}
}

@article{kumar2016fast,
  title={A fast DBSCAN clustering algorithm by accelerating neighbor searching using Groups method},
  author={Kumar, K. Mahesh and Reddy, A. Rama Mohan},
  journal={Pattern Recognition},
  volume={58},
  pages={39--48},
  year={2016},
  publisher={Elsevier}
}

@article{sinaga2020unsupervised,
  title={Unsupervised K-means clustering algorithm},
  author={Sinaga, Kristina P. and Yang, Miin-Shen},
  journal={IEEE Access},
  volume={8},
  pages={80716--80727},
  year={2020},
  publisher={IEEE}
}

@article{hou2023towards,
  title={Towards Parameter-Free Clustering for Real-World Data},
  author={Hou, Jian and Yuan, Huaqiang and Pelillo, Marcello},
  journal={Pattern Recognition},
  volume={134},
  pages={109062},
  year={2023},
  publisher={Elsevier}
}

@article{galland2005multi,
  title={Multi-component image segmentation in homogeneous regions based on description length minimization: Application to speckle, Poisson and Bernoulli noise},
  author={Galland, Fr{\'e}d{\'e}ric and Bertaux, Nicolas and R{\'e}fr{\'e}gier, Philippe},
  journal={Pattern Recognition},
  volume={38},
  number={11},
  pages={1926--1936},
  year={2005},
  publisher={Elsevier}
}

@article{haddad2013wave,
  title={Wave atoms based compression method for fingerprint images},
  author={Haddad, Zehira and Beghdadi, Azeddine and Serir, Amina and Mokraoui, Anissa},
  journal={Pattern Recognition},
  volume={46},
  number={9},
  pages={2450--2464},
  year={2013},
  publisher={Elsevier}
}

@article{uchigasaki2023deep,
  title={Deep image compression using scene text quality assessment},
  author={Uchigasaki, Shohei and Miyazaki, Tomo and Omachi, Shinichiro},
  journal={Pattern Recognition},
  volume={142},
  pages={109696},
  year={2023},
  publisher={Elsevier}
}

@article{mishra2022deep,
  title={Deep architectures for image compression: A critical review},
  author={Mishra, Dipti and Singh, Satish Kumar and Singh, Rajat Kumar},
  journal={Signal Processing},
  volume={191},
  pages={108346},
  year={2022},
  publisher={Elsevier}
}

@phdthesis{mahon2022discrete,
  title={Discrete representations of continuous data using deep learning and clustering},
  author={Mahon, Louis},
  year={2022},
  school={University of Oxford}
}

\appendix

\section{Proof of Correctness on White Noise}
In this section, we prove that the expected DDL of white noise under a GMM is a monotonically increasing function of the number of components in the GMM. 

\begin{lemma} \label{lemma:r}
When clustering white noise on $[0,1]^m$, with a $k$-component GMM, the radius $r$ of each cluster is approximated by 
 \begin{equation} \label{eq:r}
\frac{1}{\sqrt{3}\sqrt[m]{k}}\left(\frac{2}{\sqrt{3}}\right)^{1/m} \,.
\end{equation}
\end{lemma}
\begin{proof}
For balls of radius $r$, the distance along each coordinate axis between their centres will be $2r$ for one of the dimensions and $2r\frac{\sqrt{3}}{2}$ for all other dimensions. This is because the centres of each 3 touching balls will form the vertices of an equilateral triangle with side length $2r$, which will then have height $2r\frac{\sqrt{3}}{2}$. Thus, the number of balls of radius $r$ that can fit inside each axis is $\frac{1}{2r}$ for the first axis and $\frac{1}{2r}\frac{2}{\sqrt{3}}$ for all other axes. The total number of balls that can fit inside $[0,1]^m$ is, therefore 
\[
\frac{1}{(2r)^m}\left(\frac{2}{\sqrt{3}}\right)^{m-1}\,.
\]
Conversely, given that the GMM will have $k$ clusters, we can approximate the radius of each cluster using 
\begin{gather*}
\frac{1}{(2r)^m}\left(\frac{2}{\sqrt{3}}\right)^{m-1} = k \\
(2r)^m = \frac{1}{k}\left(\frac{2}{\sqrt{3}}\right)^{m-1}  \\
2r = \frac{1}{\sqrt[m]{k}}\left(\frac{2}{\sqrt{3}}\right)^{\frac{m-1}{m}}  \\
r = \frac{1}{\sqrt{3}\sqrt[m]{k}}\left(\frac{\sqrt{3}}{2}\right)^{1/m}  \\
\end{gather*}
\end{proof}

\begin{proposition} \label{proposition:m-ball-pdf}
For uniformly distributed points in an $m$-dimensional hyperball of radius $r$, the pdf of the distance of a point from the centre is given by
\[
p(x) = m \frac{x^{m-1}}{r^m}
\]
\end{proposition}
\begin{proof}
The probability density is clearly proportional to $x^{m-1}$, thus we need only to find the constant $c$ such that the pdf is appropriately normalized. 
\begin{gather*}
    c\int_{0}^r x^{m-1}dx = 1 \\
    c \frac{x^m}{m} \big |_0^r = 1\\
    c \frac{r^m}{m} = 1\\
    c = \frac{m}{r^m} \,.
\end{gather*}
\end{proof}

\begin{lemma} \label{lemma:a}
 When clustering white noise in $[0,1]^m$ dimensions, with a $k$-component GMM, the expected squared distance of a point from the centroid of its cluster, denoted $a$, is given by
 \begin{equation} \label{eq:a}
a = \frac{m}{3(m+2)}\left(\frac{\sqrt{3}}{2}\right)^{2/m}\frac{1}{k^{2/m}}\,,
\end{equation}
\end{lemma}
\begin{proof}
Using Proposition \ref{proposition:m-ball-pdf}, the expected squared distance from the centroid can be calculated directly from the pdf as
\begin{gather*}
    a = \frac{m}{r^m}\int_{0}^r x^{m+1}dx \\
    a = \frac{m}{r^m} \frac{x^{m+2}}{m+2} \big |_0^r = 1\\
    a = r^2\frac{m}{m+2}
\end{gather*}
Substituting $r$ from Lemma \ref{lemma:r} we get
\[
a = \frac{m}{3(m+2)}\left(\frac{\sqrt{3}}{2}\right)^{2/m}\frac{1}{k^{2/m}}\,,
\]
as desired.
\end{proof}

\begin{lemma} \label{lemma:variance}
When clustering white noise in $[0,1]^m$ dimensions, with a $k$-component GMM, the fit covariance matrix will $\Sigma$ will satisfy $\Sigma = \sigma I$, where 
\[
\sigma = \frac{1}{3(m+2)}\left(\frac{\sqrt{3}}{2}\right)^{2/m}\frac{1}{k^{2/m}}
\]
\end{lemma}

\begin{proof}
It is clear that $\Sigma$ will be of the form $\sigma I$ for some $\sigma$, as the clusters will all be identical by symmetry (up to approximation at the boundary of the hyperbox, but this is small for more than a few clusters). 

Next, observe that the expected squared distance of a point from the centroid of its cluster is, by linearity of expectation, equal to the sum of the expected squared distances in each coordinate, i.e. $m\sigma$. Then, by Lemma \ref{lemma:a}, we get 
\[
\sigma = \frac{1}{3(m+2)}\left(\frac{\sqrt{3}}{2}\right)^{2/m}\frac{1}{k^{2/m}}
\]
as desired.
\end{proof}

\begin{proposition} \label{proposition:fx}
The DDL of a point under the distribution of its cluster depends only on its distance from the centroid of its cluster.
\end{proposition}
\begin{proof}

The probability density of a point $z$ under a cluster with centroid $\mu$ is
\begin{align*}
p(z)    =& \frac{1}{\sqrt{2\pi |\Sigma|}}\exp{\left(\frac{-1}{2}(z-\mu)^T\Sigma^{-1}(z-\mu)\right)} \\
        =& \frac{1}{\sqrt{2\pi |\Sigma|}}\exp{\left(\frac{-|(z-\mu)|^2}{2\sigma}\right)} \,.
\end{align*}
Thus, the DDL under the cluster distribution, which we denote $\bar{D}$, is given by
\begin{gather*}
 -\ln{\left(\frac{1}{\sqrt{2\pi |\Sigma|}}\exp{\left(\frac{-|(z-\mu)|^2}{2\sigma}\right)}\right)} \\
        = \frac{1}{2} \ln{(2\pi |\Sigma|}) + \frac{|(z-\mu)|^2}{2\sigma} \,.
\end{gather*}
In what follows, we will use the function $f(x)$ to denote the DDL under the cluster distribution of a point a distance $x$ from its centroid, where
\[
f(x) = \frac{1}{2} \ln{(2\pi |\Sigma|}) + \frac{x^2}{2\sigma} \,.
\]
\end{proof}

\begin{definition}
Denote as outliers, those points with greater DDL under their cluster distribution than under the prior distribution.
\end{definition}
 
\begin{lemma} \label{lemma:d}
When clustering white noise, so that the prior distribution is $[0,1]^m$, a point is an outlier if and only if it is greater than a distance $d$ from is centroid, where $d=\sqrt{\sigma \ln{\frac{1}{2\pi |\Sigma|}}}$.
\end{lemma}
\begin{proof}
The DDL of a point treated as an outlier on the uniform $m$-box $[0,1]^m$ is 0, because treating as an outlier means using the prior distribution, which is uniform $p(x)=1$, giving DDL of $\ln{1} = 0$. Thus, a point a distance $x$ from its centroid is not an outlier if and only if its DDL under the distribution of its cluster is strictly negative:
\begin{gather*}
    DDL(x) < 0 \\
    \frac{1}{2} \ln{(2\pi |\Sigma|}) + \frac{x^2}{2\sigma} < 0 \\
    \frac{x^2}{\sigma} < -\ln{(2\pi |\Sigma|}) \\
    x^2 < \sigma \ln{\frac{1}{2\pi |\Sigma|}} \\
    x < \sqrt{\sigma \ln{\frac{1}{2\pi |\Sigma|}}} \,.
\end{gather*}
We refer to this distance $d$ as the inlier radius.
\end{proof}

\begin{lemma} \label{lemma:saturating-k}
When clustering white noise in $m$ dimensions using a GMM with $k$ components, some points will be classed as outliers if and only if $k$ satisfies
\[
k < \frac{2e\sqrt{\pi}}{\sqrt{3}}\left(\frac{e}{3(m+2)}\right)^{m/2} \\
\]
\end{lemma} \label{lemma:k-at-which-contained}
\begin{proof}
A point will be an outlier if and only if it is within the cluster, i.e. with a distance $r$ from its centroid, but outsider the inlier radius $d$. This is possible if and only if 
\begin{gather*}
    d < r \\
    \sqrt{\sigma \ln{\frac{1}{2\pi |\Sigma|}}} <\frac{1}{\sqrt{3}\sqrt[m]{k}}\left(\frac{\sqrt{3}}{2}\right)^{1/m} \\
    \sigma \ln{\frac{1}{2\pi |\Sigma|}} < \frac{1}{3k^{2/m}}\left(\frac{\sqrt{3}}{2}\right)^{2/m} \,.
\end{gather*}
As $\Sigma = \sigma I$, we can sub in $|\Sigma| = \sigma^m$, and also sub in $\sigma$ from Lemma \ref{lemma:variance}:
\begin{gather*}
    \frac{1}{3(m+2)}\left(\frac{\sqrt{3}}{2}\right)^{2/m}\frac{1}{k^{2/m}}\ln{\frac{1}{2\pi \sigma^m}} < \frac{1}{3k^{2/m}}\left(\frac{\sqrt{3}}{2}\right)^{2/m} \\
    \frac{1}{(m+2)}\ln{\frac{1}{2\pi \sigma^m}} < 1 \\
    \ln{\frac{1}{2\pi \sigma^m}} < m+2 \\
    \ln{\frac{1}{\sigma^m}} < m+2 + \ln 2{(\pi)} \\
    m\ln{\frac{1}{\sigma}} < m+2 + \ln 2{(\pi)} \\
    \ln{\frac{1}{\sigma}} < 1 + \frac{2 + \ln 2{(\pi)}}{m}   \\
    \sigma > \frac{1}{e^{1 + \frac{2 + \ln 2{(\pi)}}{m}}}   \\
    \frac{1}{3(m+2)}\left(\frac{\sqrt{3}}{2}\right)^{2/m}\frac{1}{k^{\frac{2m-2}{m}}} > \frac{1}{\sqrt[m]{\pi}e^{1 + \frac{2}{m}}}   \\
    \frac{\sqrt[m]{\pi}e^{1 + \frac{2}{m}}}{3(m+2)}\left(\frac{\sqrt{3}}{2}\right)^{2/m} > k^{2/m} \\
    k^{2/m} < \frac{\sqrt[m]{\pi}e^{1 + \frac{2}{m}}}{3(m+2)}\left(\frac{\sqrt{3}}{2}\right)^{2/m} \\
    k < \frac{\sqrt{\pi}e^{\frac{m}{2} + 1}}{(3(m+2))^{m/2}}\left(\frac{2}{\sqrt{3}}\right) \\
    k < \frac{\sqrt{\pi}e^{\frac{m}{2} + 1}}{(3(m+2))^{m/2}}\frac{\sqrt{3}}{2} \\
    k < \frac{\sqrt{\pi}e\sqrt{3}}{2}\frac{e^{m/2}}{(3(m+2))^{m/2}} \\
    k < \frac{e\sqrt{3\pi}}{2}\left(\frac{e}{3(m+2)}\right)^{m/2} \\
\end{gather*}
\end{proof}

\begin{lemma} \label{lemma:form-of-EDDL}
For a $k$-component GMM fit on $m$-dimensional white noise, the expected DDL of a point is 
\[
\frac{m}{r^m}\int_0^x x^{m-1} \min(\{0,f(x)\}) dx + \ln{k} \,.
\]
with $f(x)$ defined as in Proposition \ref{proposition:fx}.
\end{lemma}
\begin{proof}
The DDL of a point can be decomposed as the number of bits needed to specify which cluster the point belongs to, plus the DDL of the point under that cluster. The former is equal to $\ln{k}$, because each cluster will, by symmetry, be equally sized. The latter, we denote $\bar{D}$, and is given by
\[
\bar{D} = \int_0^r p(r) \min(\{0,f(x)\}) \,,
\]
where $p(r)$ is the probability density function of the distance of a point from its centroid and $f(x)$ specifies the DDL of a point in terms of the distance from its centroid. 

Substituting for $p(x)$ from Proposition \ref{proposition:m-ball-pdf} gives 
\[
\bar{D} = \frac{m}{r^m}\int_0^x x^{m-1} \min(\{0,f(x)\}) dx \,.
\]
\end{proof}

\begin{lemma} \label{lemma:increasing-when-outliers}
When clustering white noise on $[0,1]^m$ with a $k$-component GMM, for $k \leq d$, the expected DDL is an increasing function of $k$. 
\end{lemma}
\begin{proof}
When some points are outliers, we have, using Lemma \ref{lemma:form-of-EDDL}
\begin{align*}
    \bar{D} &= \frac{m}{r^m}\int_0^r x^{m-1} \min(\{0,f(x)\}) dx \\
            &= \frac{m}{r^m}\int_{\{x \in [0,r] | f(x) < 0\}} r^{m-1} f(x) dx \,.
\end{align*}
By Lemma \ref{lemma:d}, $\{x \in [0,r] | f(x) < 0\} = [0,d)$, giving
\begin{align*}
     \bar{D} &= \frac{m}{r^m}\int_0^d x^{m-1} f(x) dx \\
\end{align*}
Substituting $f(x)$ from Proposition \ref{proposition:fx} and integrating:
\begin{align*}
    \bar{D} &= \frac{m}{r^m}\int_0^d x^{m-1} \left(\frac{1}{2} \ln{(2\pi |\Sigma|}) + \frac{x^2}{2\sigma}\right)  dx \\
            &= \frac{m}{r^m}\int_0^d \frac{1}{2} \ln{(2\pi |\Sigma|})x^{m-1}  + \frac{x^{m+1}}{2\sigma}  dx \\
            &= \frac{m}{r^m} \left(d^{m} \frac{1}{2m} \ln{(2\pi |\Sigma|}) + \frac{d^{m+2}}{2\sigma(m+2)} \right) \\ 
            &= \frac{d^{m}}{2r^m}  \left(\ln{(2\pi |\Sigma|}) + d^{2}\frac{m}{\sigma(m+2)} \right) \,.
\end{align*}
Substituting $d$ from Lemma \ref{lemma:d}:
\begin{align*}
     \bar{D} &= \frac{(\sigma \ln{\frac{1}{2\pi |\Sigma|}})^{m/2}}{2r^m}  \left(\ln{(2\pi |\Sigma|}) + (\sigma \ln{\frac{1}{2\pi |\Sigma|}})\frac{m}{\sigma(m+2)} \right) \\
            &= \frac{(\sigma \ln{\frac{1}{2\pi |\Sigma|}})^{m/2}}{2r^m}  \left(\ln{(2\pi |\Sigma|}) + (\ln{\frac{1}{2\pi |\Sigma|}})\frac{m}{m+2} \right) \\
            &= \frac{(\sigma \ln{\frac{1}{2\pi |\Sigma|}})^{m/2}}{2r^m}  \left(\ln{\frac{1}{2\pi |\Sigma|}}\frac{-2}{m+2} \right) \\
            &= \frac{-\sigma ^{m/2}}{r^m(m+2)}  \left(\ln{\frac{1}{2\pi |\Sigma|}} \right)^{\frac{m}{2}+1} \,.
\end{align*}
Substituting $\sigma$ from Lemma \ref{lemma:variance}:
\begin{align*}
      \bar{D} &= -\left(\frac{1}{3(m+2)}\left(\frac{\sqrt{3}}{2}\right)^{2/m}\frac{1}{k^{2/m}}\right)^{m/2} \frac1{r^m(m+2)} \left(\ln{\frac{1}{2\pi |\Sigma|}} \right)^{\frac{m}{2}+1} \\
            &= -\left(\frac{1}{(3(m+2))^{m/2}}\left(\frac{\sqrt{3}}{2}\right)\frac{1}{k}\right) \frac{1}{r^m(m+2)} \left(\ln{\frac{1}{2\pi |\Sigma|}} \right)^{\frac{m}{2}+1} \\
            &= -\left(\frac{1}{(3(m+2))^{m/2}}\right) \left(\frac{\sqrt{3}}{2}\right) \frac{1}{kr^m(m+2)} \left(\ln{\frac{1}{2\pi |\Sigma|}} \right)^{\frac{m}{2}+1} \\
            &=  \frac{-3 \sqrt{3}}{r^m2(3)^{m/2}k(m+2)} \left(\frac{1}{(m+2)}\ln{\frac{1}{2\pi |\Sigma|}} \right)^{\frac{m}{2}+1} \,.
\end{align*}
Substituting $r$ from Lemma \ref{lemma:r}:
\begin{align*}
       \bar{D} &= -\left(\sqrt{3}\sqrt[m]{k}\left(\frac{2}{\sqrt{3}}\right)^{1/m} \right)^m\frac{-3 \sqrt{3}}{2(3)^{m/2}k(m+2)} \left(\frac{1}{(m+2)}\ln{\frac{1}{2\pi |\Sigma|}} \right)^{\frac{m}{2}+1} \\
       \bar{D} &= -(\sqrt{3})^m k\frac{2}{\sqrt{3}} \frac{-3 \sqrt{3}}{2(3)^{m/2}k(m+2)} \left(\frac{1}{(m+2)}\ln{\frac{1}{2\pi |\Sigma|}} \right)^{\frac{m}{2}+1} \\
       \bar{D} &= \frac{-3}{(m+2)^{\frac{m}{2}+2}} \left(\ln{\frac{1}{2\pi |\Sigma|}} \right)^{\frac{m}{2}+1} \,.
\end{align*}
Substituting $\Sigma=\sigma I$, with $\sigma$ from Lemma \ref{lemma:variance}:
\begin{align*}
       \bar{D} &= \frac{-3}{(m+2)^{\frac{m}{2}+2}} \left(\ln{\frac{1}{2\pi}} + \ln{\frac{1}{|\Sigma|}} \right)^{\frac{m}{2}+1} \\
            &= \frac{-3}{(m+2)^{\frac{m}{2}+2}} \left(\ln{\frac{1}{2\pi}} + \ln{\frac{1}{\sigma^m}} \right)^{\frac{m}{2}+1} \\
            &= \frac{-3}{(m+2)^{\frac{m}{2}+2}} \left(\ln{\frac{1}{2\pi}} + \ln{\left(3^m(m+2)^m\frac{4}{3}k^2\right)} \right)^{\frac{m}{2}+1} \\
            &= \frac{-3}{(m+2)^{\frac{m}{2}+2}} \left((m)\ln{3} +m\ln{(m+2)} + \ln{\tfrac{4}{3}} + 2 \ln{k} - \ln{2\pi} \right)^{m/2} \\
            &= \frac{-3}{(m+2)^{\frac{m}{2}+2}} \left(m\ln{(3m+6)} + 2 \ln{k} + \ln{\tfrac{2}{3\pi}} \right)^{\frac{m}{2}+1} \\
            &= \frac{-3}{(m+2)^{\frac{m}{2}+2}} \left(m\ln{(3m+6)} + 2 \ln{k} + \ln{\tfrac{2}{3\pi}} \right)^{\frac{m}{2}+1} \\
\end{align*}
Differentiating, with respect to $k$, the total DDL of $\bar{D} + \ln{k}$:
\begin{align*}
    &\frac{d}{dk}\left(\frac{-3}{(m+2)^{\frac{m}{2}+2}}(m \ln{(3m+6)}+\ln{\tfrac{2}{3\pi}}+2\ln{k})^{\frac{m}{2}+1} + \ln{k}\right) \\
    &=\frac{-3}{(m+2)^{\frac{m}{2}+2}}\frac{d}{dk}\left((m \ln{(3m+6)}+\ln{\tfrac{2}{3\pi}}+2\ln{k})^{\frac{m}{2}+1}\right) + \frac{1}{k} \\
    &=\frac{-3}{(m+2)^{\frac{m}{2}+2}}\left((\frac{m}{2}+1)(m \ln{(3m+6)}+\ln{\tfrac{2}{3\pi}}+2\ln{k})^{m/2}\frac{2}{k}\right) + \frac{1}{k} \\
    &=\frac{-3}{(m+2)^{\frac{m}{2}+2}}\frac{(m+2)}{k}\left(m \ln{(3m+6)}+\ln{\tfrac{2}{3\pi}}+2\ln{k}\right)^{m/2} + \frac{1}{k} \\
    &=\frac{-3}{k(m+2)^{\frac{m}{2}+1}}\left(m \ln{(3m+6)}+\ln{\tfrac{2}{3\pi}}+2\ln{k}\right)^{m/2} + \frac{1}{k} \,.
\end{align*}
We want to show this derivative is positive:
\begin{align*}
    &\frac{-3}{k(m+2)^{\frac{m}{2}+1}}^{\frac{m}{2}+2}\left(m \ln{(3m+6)}+\ln{\tfrac{2}{3\pi}}+2\ln{k}\right)^{m/2} + \frac{1}{k} > 0\\
    &\iff \frac{1}{k} > \frac{3}{k(m+2)^{\frac{m}{2}+1}}\left(m \ln{(3m+6)}+\ln{\tfrac{2}{3\pi}}+2\ln{k}\right)^{m/2}\\
    &\iff \frac{3}{(m+2)^{\frac{m}{2}+1}}\left(m \ln{(3m+6)}+\ln{\tfrac{2}{3\pi}}+2\ln{k}\right)^{m/2} < 1\\
\end{align*}
As, by assumption, some points are outliers, we can use the upper bound on $k$ from Lemma \ref{lemma:k-at-which-contained}:
\begin{align*}
    &\frac{3}{(m+2)^{\frac{m}{2}+1}}\left(m \ln{(3m+6)}+\ln{\tfrac{2}{3\pi}}+2\ln{k}\right)^{m/2} \\
    &= \frac{3}{(m+2)^{\frac{m}{2}+1}}\left(m \ln{(3m+6)}+\ln{\tfrac{2}{3\pi}}+\ln{k^2}\right)^{m/2} \\
    &\leq \frac{3}{(m+2)^{\frac{m}{2}+1}}\left(m \ln{(3m+6)}+\ln{\tfrac{2}{3\pi}}+\ln{\left(\frac{4\pi e^2}{3}\left(\frac{e}{3(m+2)}\right)^m\right)}\right)^{m/2} \\
    &= \frac{3}{(m+2)^{\frac{m}{2}+1}}\left(m \ln{(3m+6)} + \ln{\tfrac{2}{3\pi}} + 2 + \ln{\tfrac{4\pi}{3}} + m - m\ln{(3m+6)}\right)^{m/2} \\
    &= \frac{3}{(m+2)^{\frac{m}{2}+1}}\left(2 + \ln{\tfrac{2}{3\pi}} + \ln{\tfrac{4\pi}{3}} + m\right)^{m/2} \\
    &= \frac{3}{(m+2)^{\frac{m}{2}+1}}\left(m + 2 + \ln{\tfrac{8}{9}}\right)^{m/2} \\
    &< \frac{3}{(m+2)^{\frac{m}{2}+1}}(m + 2)^{m/2} \\
    &= \frac{3}{m+2} \leq 1 \,,
\end{align*}
where the last inequality holds because the dimension $m$, is $\geq 1$.
\end{proof}

\begin{lemma} \label{lemma:constant-when-no-outliers}
When clustering white noise on $[0,1]^m$ with a $k$-component GMM, for $k>d$, the expected DDL is independent of $k$. 
\end{lemma}
\begin{proof}
When no points are outliers, $f(x) = g(x)$ for all $x$, so, using Lemma \ref{lemma:form-of-EDDL}
\begin{align*}
    \bar{D} &= \frac{m}{r^m}\int_0^r x^{m-1} f(x) dx \\
            &= \frac{m}{r^m}\int_0^r x^{m-1} \left(\frac{1}{2} \ln{(2\pi |\Sigma|}) + \frac{x^2}{2\sigma}\right)  dx \\
            &= \frac{m}{r^m} \left(\frac{r^m}{2m} \ln{(2\pi |\Sigma|}) + \frac{r^{m+2}}{2\sigma(m+2)}\right) \\
            &= \frac{1}{2} \ln{(2\pi |\Sigma|}) + \frac{r^2 m}{\sigma(m+2)} \,.
\end{align*}
Substituting $r$ from Lemma \ref{lemma:r}:
\begin{align*}
    \bar{D} &= \frac{1}{2} \left(\ln{(2\pi |\Sigma|}) + \frac{1}{3k^{2/m}}\left(\frac{\sqrt{3}}{2}\right)^{2/m}\frac{m}{\sigma(m+2)}\right)  dx \\
\end{align*}
Substituting $\sigma$ from Lemma \ref{lemma:variance}:
\begin{align*}
    \bar{D} &= \frac{1}{2} \left(\ln{(2\pi |\Sigma|}) + \frac{1}{3k^{2/m}}\left(\frac{\sqrt{3}}{2}\right)^{2/m}3(m+2)\left(\frac{2}{\sqrt{3}}\right)^{2/m}k^{2/m}\frac{m}{m+2}\right)  dx \\
    \bar{D} &= \frac{1}{2} \left(\ln{(2\pi |\Sigma|}) + m \right) \,.
\end{align*}
Substituting $\Sigma=\sigma I$, with $\sigma$ from Lemma \ref{lemma:variance}:
\begin{align*}
    \bar{D} &= \frac{1}{2} \left(\ln{(2\pi |\Sigma|}) + m \right) \\
            &= \frac{1}{2} \left(\ln{(2\pi)} + m \ln{{\sigma}} + m \right) \\
            &= \frac{1}{2} \left(\ln{(2\pi)} + m \ln{\left(\frac{1}{3(m+2)}\left(\frac{\sqrt{3}}{2}\right)^{2/m}\frac{1}{k^{2/m}}\right)} + m \right) \\
            &= \frac{1}{2} \left(\ln{(2\pi)} - m \ln{\left(3(m+2)\left(\frac{2}{\sqrt{3}}\right)^{2/m}k^{2/m}\right)} + m \right) \\
            &= \frac{1}{2} \left(\ln{(2\pi)} - m \ln{(3m+6)} + 2\ln{\tfrac{2}{\sqrt{3}}} -2\ln{k} + m \right) \\
            &= \frac{1}{2} \left(\ln{\frac{4\pi}{\sqrt{3}}}+ m \ln{\left(1+\frac{1}{3m+6}\right)} - 2\ln{k} + m \right) \\
\end{align*}
The full DDL is then
\begin{align*}
    & \bar{D} + \ln{k} \\
    &= \frac{1}{2} \left(\ln{\frac{4\pi}{\sqrt{3}}}+ m \ln{\left(1+\frac{1}{3m+6}\right)} - 2\ln{k} + m \right)  + \ln{k} \\
    &= \frac{1}{2} \left(\ln{\frac{4\pi}{\sqrt{3}}}+ m \ln{\left(1+\frac{1}{3m+6}\right)} + m \right) \\
\end{align*}
\end{proof}

\begin{theorem}
When clustering white noise in $[0,1]^m$, using a GMM with $k$ components, the expected DDL of a point is a monotonically increasing function of $k$.
\end{theorem}
\begin{proof}
By Lemma \ref{lemma:increasing-when-outliers}, the expected DDL is strictly increasing in $k$ up to $k=d$. By Lemma \ref{lemma:constant-when-no-outliers}, the expected DDL is constant in $k$ for $k > d$.
\end{proof}

\section{Worked Examples}
Section \ref{subsec:worked-example} presented a single worked example of my image complexity metric on a single, randomly chosen image from ImageNet. This section contains further examples from ImageNet and others from Cifar, MNIST, and DTD2.

\begin{figure*}
    \centering
    \begin{tabular}{@{}c@{}c@{}}
    \includegraphics[width=0.53\textwidth]{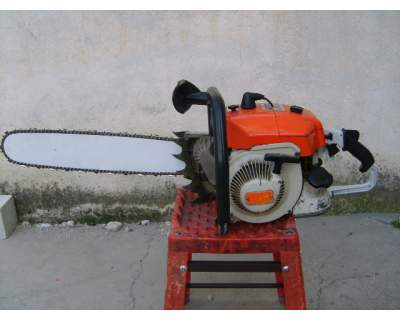} & 
    \end{tabular}
    \caption{ \small  \small Example of a relatively high-resolution real-world image from imagenet. ID: n03445777\_10762.} \label{fig:imgs/im-chainsaw-image}
\end{figure*}
\noindent \begin{large} \textbf{Layer 1} \end{large} \\
Number of points to be clustered (pixels): 50246 \\
Number of components found by MDL, as per \eqref{eq:k-star}: 8 \\
Assign each pixel a label from $0, \dots, 6$, and form patch signatures as multisets of labels inside all $4 \times 4$ patches, which gives 48450 patches, of 2628 different unique values. \\
Entropy of resulting categorical distribution of patch signatures: \textbf{9.280} \\

\noindent \begin{large} \textbf{Layer 2} \end{large} \\
Number of points to be clustered: 48450 \\
Number of components found by MDL, as per \eqref{eq:k-star}: 6 \\
Assign each point a label from $0, \dots, 7$, and form patch signatures as multisets of labels inside all $8 \times 8$ patches, which gives 44954 patches, of 3809 different unique values. \\
Entropy of resulting categorical distribution of patch signatures: \textbf{9.363} \\

\noindent \begin{large} \textbf{Layer 3} \end{large} \\
Number of points to be clustered: 44954 \\
Number of components found by MDL, as per \eqref{eq:k-star}: 8 \\
Assign each point a label from $0, \dots, 7$, and form patch signatures as multisets of labels inside all $16 \times 16$ patches, which gives 38346 patches, of 6956 different unique values. \\
Entropy of resulting categorical distribution of patch signatures: \textbf{12.307} \\

\noindent \begin{large} \textbf{Layer 4} \end{large} \\
Number of points to be clustered: 38346 \\
Number of components found by MDL, as per \eqref{eq:k-star}: 7 \\
Assign each point a label from $0, \dots, 6$, and form patch signatures as multisets of labels inside all size $32 \times 32$ patches, which gives 26666 patches, of 5325 different unique values. \\
Entropy of resulting categorical distribution of patch signatures: \textbf{12.378} \\
\\
\begin{large}
Total complexity: $7.995 + 10.194 + 12.772 + 12.753 = \boldsymbol{43.314}$
\\
\end{large}

\subsection{Cifar Car Image}
\begin{figure*}[h]
    \centering
    \begin{tabular}{@{}c@{}c@{}}
    \includegraphics[width=0.53\textwidth]{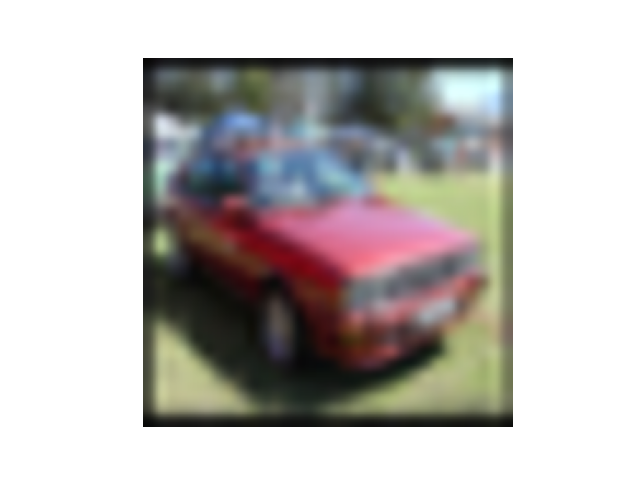} & 
    \end{tabular}
    \caption{ \small Example of a low-resolution, real-world image from CIFAR.}
    \label{fig:imgs/cifar-car-image}
\end{figure*}
\noindent \begin{large} \textbf{Layer 1} \end{large} \\
Num points to be clustered (pixels): 50176 \\
Num components found by MDL: 7 \\
Assign each pixel a label from $0, \dots, 6$, and form patch signatures as multisets of labels inside all patches of size $4 \times 4$. \\
Num patch signatures: 48400 \\
Num unique patch signatures: 304 \\
Entropy of categorical distribution of patch signatures: \textbf{5.6} \\

\noindent \begin{large} \textbf{Layer 2} \end{large} \\
Num points to be clustered: 48400 \\
Num components found by MDL: 6 \\
Assign each point a label from $0, \dots, 5$, and form patch signatures as multisets of labels inside all patches of size $8 \times 8$. \\
Num patch signatures: 44944 \\
Num unique patch signatures: 2615 \\
Entropy of categorical distribution of patch signatures: \textbf{8.67} \\

\noindent \begin{large} \textbf{Layer 3} \end{large} \\
Num points to be clustered: 44944 \\
Num components found by MDL: 8 \\
Assign each pixel a label from $0, \dots, 7$, and form patch signatures as multisets of labels inside all patches of size $16 \times 16$. \\
Num patch signatures: 38416 \\
Num unique patch signatures: 6455 \\
Entropy of categorical distribution of patch signatures: \textbf{11.69} \\

\noindent \begin{large} \textbf{Layer 4} \end{large} \\
Num points to be clustered: 38212 \\
Num components found by MDL: 8 \\
Assign each pixel a label from $0, \dots, 7$, and form patch signatures as multisets of labels inside all patches of size $32 \times 32$. \\
Num patch signatures: 26896 \\
Num unique patch signatures: 6445 \\
Entropy of categorical distribution of patch signatures: \textbf{12.61} \\
\\
\begin{large}
\\
Total complexity: $5.61 + 8.67 + 11.69 + 12.61 = \boldsymbol{38.50}$
\end{large}
\clearpage

\subsection{DTD Fine Woven Texture}
\begin{figure*}[h]
    \centering
    \begin{tabular}{@{}c@{}c@{}}
    \includegraphics[width=0.53\textwidth]{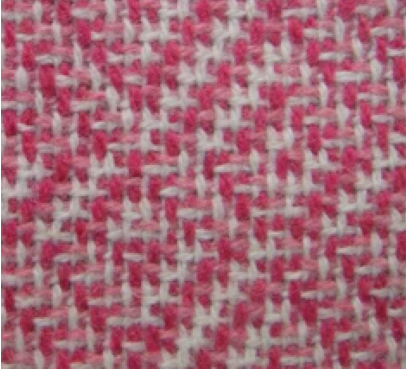} & 
    \end{tabular}
    \caption{ \small Example of a detailed, high-resolution, repetitive pattern from DTD2.}
    \label{fig:imgs/dtd-fine-woven-image}
\end{figure*}
\noindent \begin{large} \textbf{Layer 1} \end{large} \\
Num points to be clustered (pixels): 50290 \\
Num components found by MDL: 8 \\
Assign each pixel a label from $0, \dots, 7$, and form patch signatures as multisets of labels inside all patches of size $4 \times 4$. \\
Num patch signatures: 48400 \\
Num unique patch signatures: 7568 \\
Entropy of categorical distribution of patch signatures: \textbf{12.16} \\

\noindent \begin{large} \textbf{Layer 2} \end{large} \\
Num points to be clustered: 48400 \\
Num components found by MDL: 8 \\
Assign each point a label from $0, \dots, 7$, and form patch signatures as multisets of labels inside all patches of size $8 \times 8$. \\
Num patch signatures: 44944 \\
Num unique patch signatures: 13356 \\
Entropy of categorical distribution of patch signatures: \textbf{13.70} \\

\noindent \begin{large} \textbf{Layer 3} \end{large} \\
Num points to be clustered: 44944 \\
Num components found by MDL: 1 \\
Assign each pixel a label from $0, \dots, 7$, and form patch signatures as multisets of labels inside all patches of size $16 \times 16$. \\
Num patch signatures: 36416 \\
Num unique patch signatures: 1 \\
Entropy of categorical distribution of patch signatures: \textbf{0} \\

\noindent \begin{large} \textbf{Layer 4} \end{large} \\
All patch signatures are identical because only one cluster was found at the previous level. So the entropy is \textbf{0}.
\begin{large}
\\
\\
Total complexity: $12.16 + 13.70 + 0 + 0 = \boldsymbol{25.87}$
\end{large}
\clearpage

\section{Datasets: Further Details}
\subsection{Synthetic Datasets}
As described in Section \ref{sec:experimental-evaluation}, I created three synthetic datasets to help test my image complexity metric on a variety of images. The experimental results I report use 500 images sampled from these synthetic datasets. Here, I give the full details for the creation of these datasets. Code will also be released on publication.

\paragraph{Stripes}
These images depict a repeated striped black-and-white pattern. The thickness of the lines, in pixels, is sampled uniformly at random from $[3,10]$, and the slope of the lines is sampled uniformly at random from $[-0.5, -1.5]$. It is sufficient to consider negative slopes only as our method, and all methods that I compare to, are invariant to reflections, so the striped images with slope in $[0.5, 1.5]$ would receive identical scores to those in $[-0.5, -1.5]$. Note that our method is not necessarily invariant to rotations, because it is based on square, axis-aligned patches of pixels. The same is true of the fractal dimension computed with the Minkowski-Bouligand dimension (i.e., the fractal dimension), as it uses a box-counting method. Examples of Stripes images are shown in Figure \ref{fig:imgs/stripes-images}.
\begin{figure}[h]
    \centering
    \begin{tabular}{@{}c@{}c@{}}
    \includegraphics[width=0.53\textwidth]{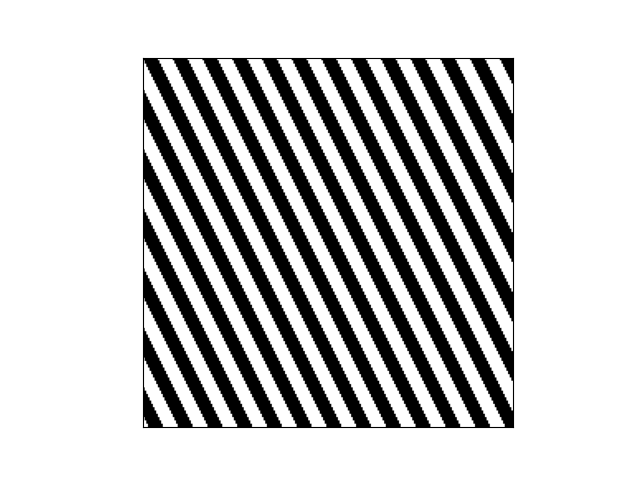} & 
    \includegraphics[width=0.53\textwidth]{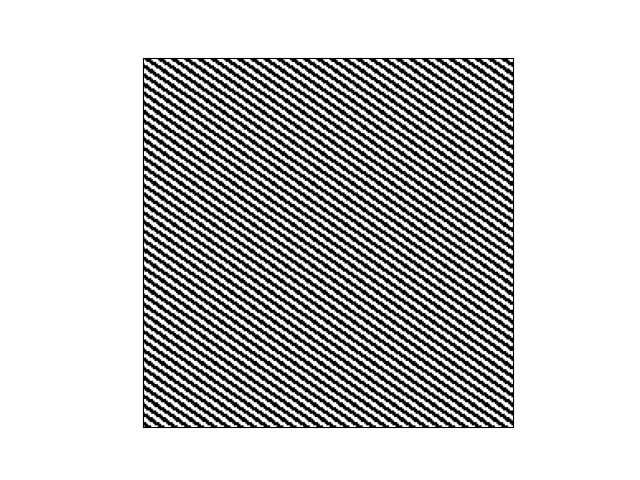}\\
    \includegraphics[width=0.53\textwidth]{imgs/stripes_image4.png} & 
    \includegraphics[width=0.53\textwidth]{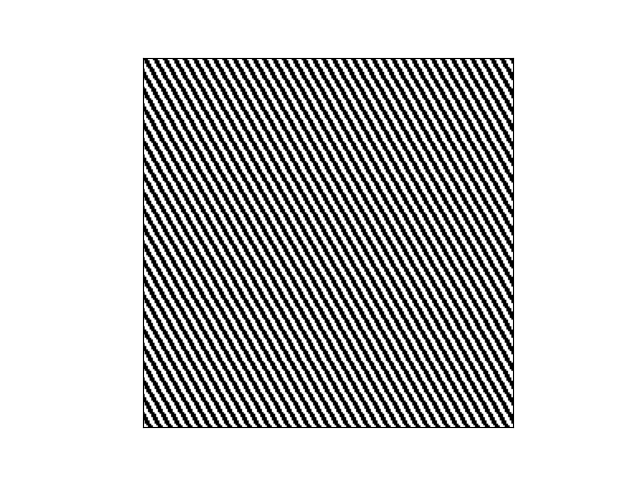}\\
     \includegraphics[width=0.53\textwidth]{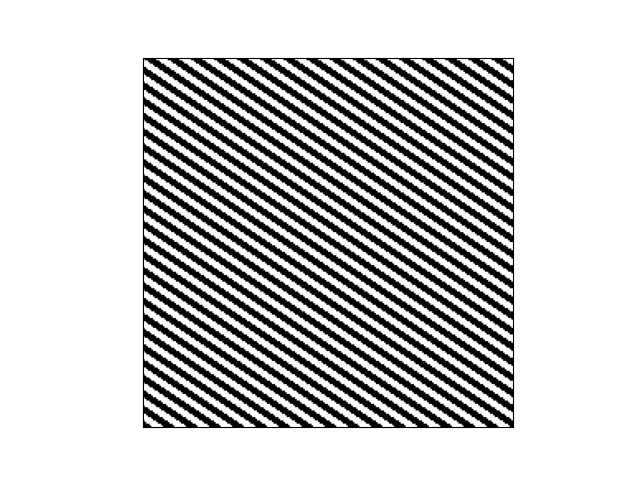} & 
    \includegraphics[width=0.53\textwidth]{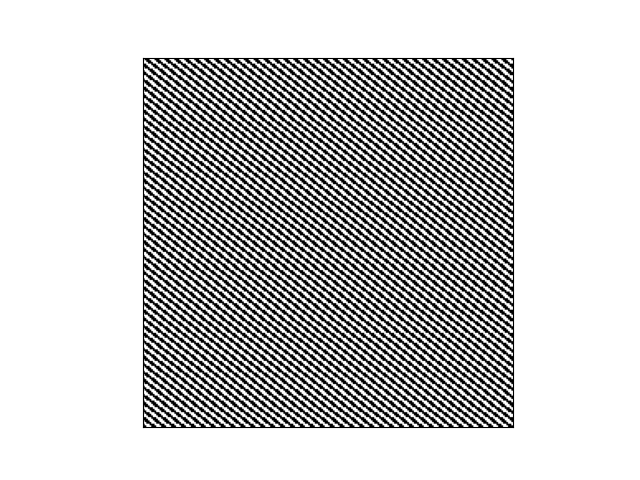}\\
    \end{tabular}
    \caption{ \small Examples of images from our synthetic Stripes dataset. Most existing methods assign theses images a high complexity. Ours assigns them low, but non-zero complexity.}
    \label{fig:imgs/stripes-images}
\end{figure}

\paragraph{Halves}
These images have one half entirely black and the other entirely white, with the dividing line being at various angles. As with Stripes, the slope of this dividing line is sampled uniformly at random from $[-0.5, -1.5]$. Examples are shown in Figure \ref{fig:halves-images}.
\begin{figure}[h]
    \centering
    \begin{tabular}{@{}c@{}c@{}}
    \includegraphics[width=0.53\textwidth]{imgs/halves_image0.png} & 
    \includegraphics[width=0.53\textwidth]{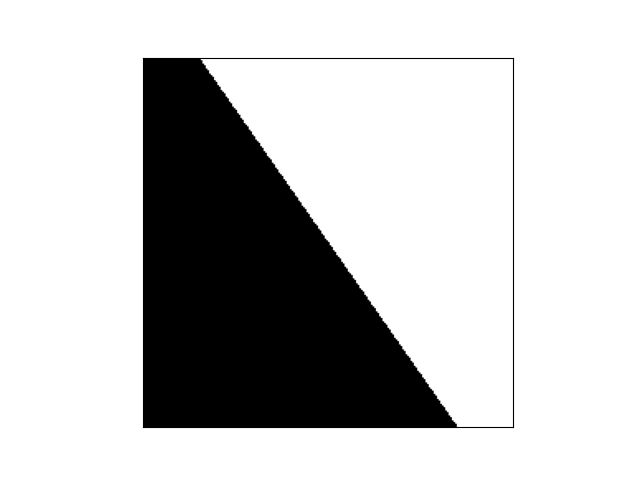}\\
    \includegraphics[width=0.53\textwidth]{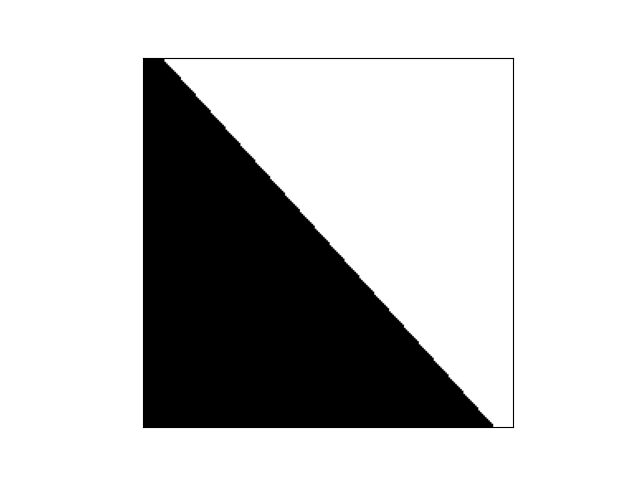} & 
    \includegraphics[width=0.53\textwidth]{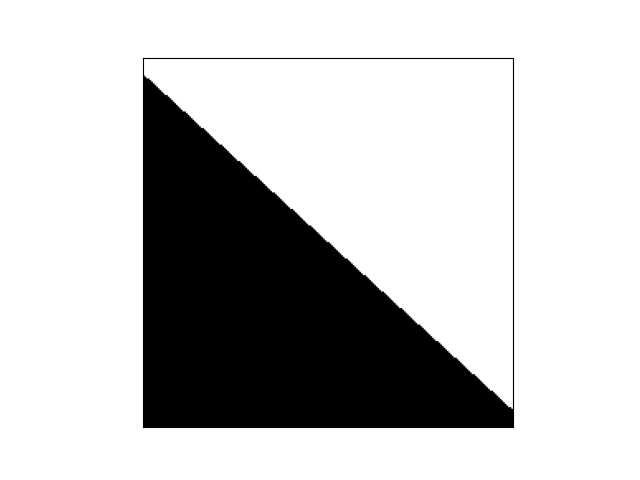}\\
     \includegraphics[width=0.53\textwidth]{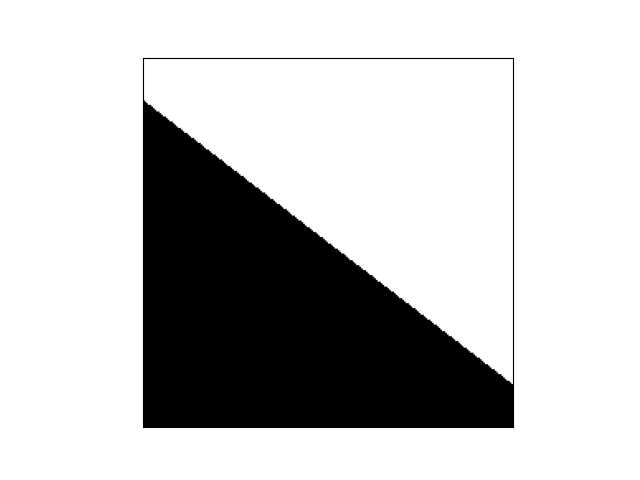} & 
    \includegraphics[width=0.53\textwidth]{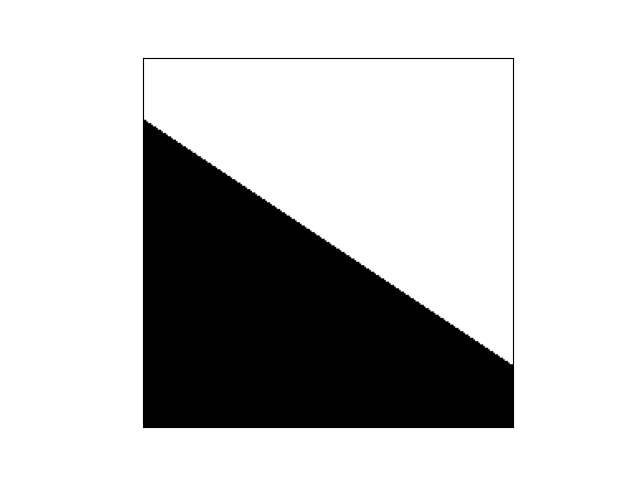}\\
    \end{tabular}
    \caption{ \small Examples of images from our synthetic Halves dataset. Our method assigns these low complexity as do existing methods. However, when we break the method down by scale, as discussed in Section 4.2, we see that it assigns some complexity at a high scale, more so than, e.g., Stripes, because there is some difference between different parts at a high scale, whereas in Stripes, both halves of each image are the same.}
    \label{fig:halves-images}
\end{figure}

\paragraph{Rand}
These images are white noise. Their values are sampled uniformly at random from $[0,1]$, independently for each location and each of three colour channels. Examples are shown in Figure \ref{fig:imgs/rand-images}.

\begin{figure*}[h]
    \centering
    \begin{tabular}{@{}c@{}c@{}}
    \includegraphics[width=0.53\textwidth]{imgs/rand_image0.png} & 
    \includegraphics[width=0.53\textwidth]{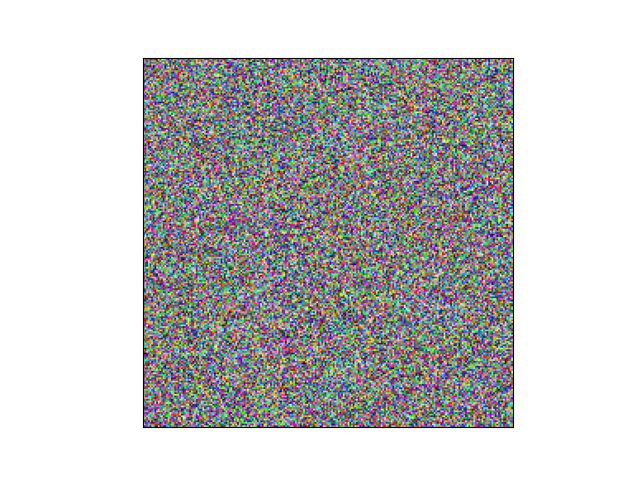}\\
    \includegraphics[width=0.53\textwidth]{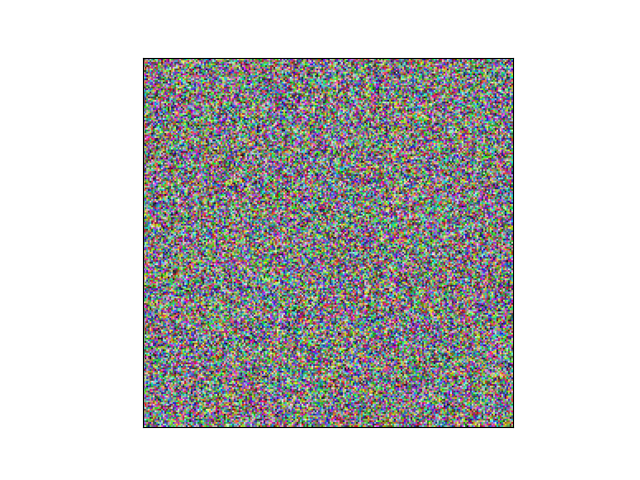} & 
    \includegraphics[width=0.53\textwidth]{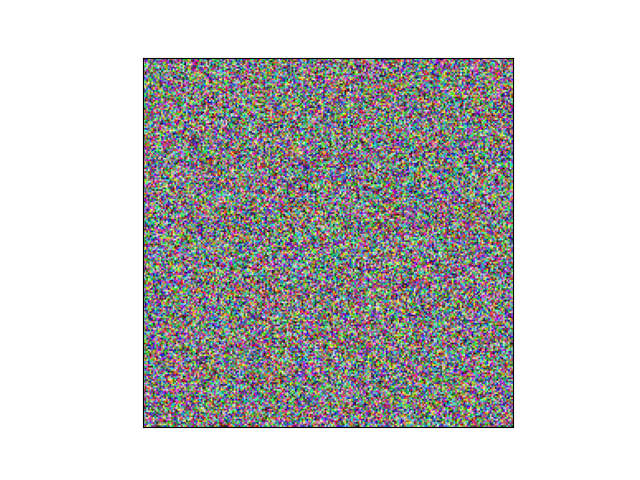}\\
     \includegraphics[width=0.53\textwidth]{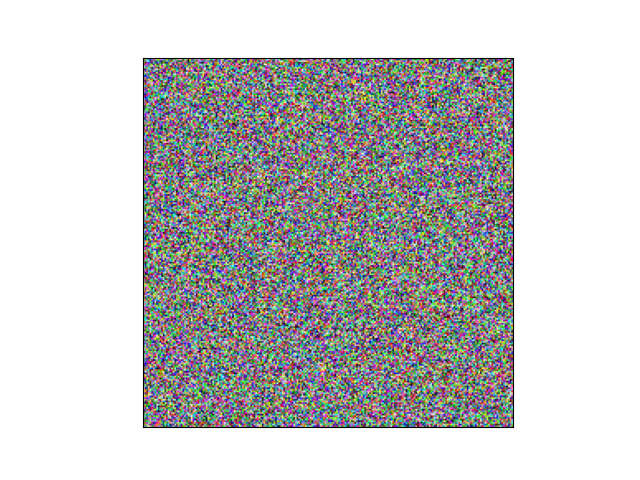} & 
    \includegraphics[width=0.53\textwidth]{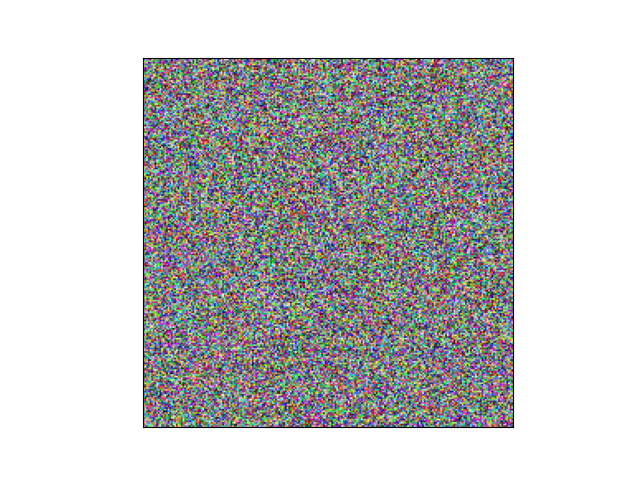}\\
    \end{tabular}
    \caption{ \small Examples of the white noise images I used in the Rand dataset. Existing complexity measures assign these images a very high complexity. Our method, in contrast, gives them all a zero complexity.}
    \label{fig:imgs/rand-images}
\end{figure*}

\clearpage

\subsection{Existing Datasets}
For the experiments in Section 4, I randomly sample 500 images from each of ImageNet, CiFAR, and MNIST. For DTD2, we manually search through all 5640 images in the original Describable Textures Dataset \cite{cimpoi2014describing}, and we find 341 images with fine detailed but repetitive textures. This section contains further information on the images used for each dataset.

\paragraph{ImageNet}
All images are from the Imagenette subset of ImageNet. The list of labels that we present is taken from the Imagenette labels, available at https://s3.amazonaws. com/fast-ai-imageclas/imagenette2.tgz. Most of the images IDs are of the form <wordnet-synset-id> - <index-within-synset>. Some instead use the class label from the 2012 version of ILSVRC.

\begin{multicols}{3}

\begin{itemize}

\item n03394916 - 50642.
\item n02979186 - 10250.
\item n03028079 - 14492.
\item n03888257 - 36631.
\item n03445777 - 3291.
\item n03425413 - 14940.
\item n03417042 - 19472.
\item n03000684 - 9440.
\item n01440764 - 16090.
\item ILSVRC2012-00023440.
\item n03394916 - 39102.
\item n02979186 - 24592.
\item n03028079 - 25712.
\item n03888257 - 16542.
\item n03445777 - 12262.
\item n03425413 - 17220.
\item n03417042 - 3351.
\item n03000684 - 32351.
\item n01440764 - 21191.
\item n02102040 - 1782.
\item n03394916 - 46700.
\item n02979186 - 17680.
\item n03028079 - 7422.
\item n03888257 - 25150.
\item n03445777 - 5582.
\item n03425413 - 8801.
\item n03417042 - 25411.
\item n03000684 - 20052.
\item n01440764 - 5361.
\item n02102040 - 2930.
\item n03394916 - 43381.
\item n02979186 - 20620.
\item n03028079 - 16811.
\item n03888257 - 21201.
\item n03445777 - 11171.
\item n03425413 - 13581.
\item n03417042 - 6420.
\item ILSVRC2012-00045501.
\item n01440764 - 16192.
\item n02102040 - 5890.
\item n03394916 - 43532.
\item n02979186 - 9910.
\item ILSVRC2012-00004912.
\item n03888257 - 22330.
\item n03445777 - 10762.
\item n03425413 - 11061.
\item n03417042 - 9620.
\item n03000684 - 16872.
\item n01440764 - 6812.
\item n02102040 - 3452.
\item n03394916 - 32870.
\item n02979186 - 23650.
\item n03028079 - 27781.
\item n03888257 - 20300.
\item n03445777 - 6091.
\item n03425413 - 20371.
\item n03417042 - 3771.
\item n03000684 - 10992.
\item n01440764 - 13842.
\item n02102040 - 6851.
\item n03394916 - 11582.
\item n02979186 - 18720.
\item n03028079 - 29062.
\item n03888257 - 18441.
\item n03445777 - 520.
\item n03425413 - 21180.
\item n03417042 - 4072.
\item n03000684 - 9452.
\item n01440764 - 14150.
\item n02102040 - 652.
\item n03394916 - 36000.
\item n02979186 - 1810.
\item n03028079 - 4612.
\item n03888257 - 7921.
\item n03445777 - 5932.
\item n03425413 - 12711.
\item n03417042 - 4462.
\item n03000684 - 661.
\item n01440764 - 9152.
\item n02102040 - 7942.
\item n03394916 - 36172.
\item n02979186 - 13740.
\item n03028079 - 9920.
\item n03888257 - 36390.
\item n03445777 - 6162.
\item n03425413 - 12951.
\item n03417042 - 2150.
\item n03000684 - 10690.
\item n01440764 - 14342.
\item n02102040 - 5942.
\item n03394916 - 43422.
\item n02979186 - 13442.
\item n03028079 - 34051.
\item n03888257 - 46870.
\item n03445777 - 13480.
\item n03425413 - 20751.
\item n03417042 - 5920.
\item n03000684 - 18020.
\item n01440764 - 7982.
\item n02102040 - 182.
\item n03394916 - 42721.
\item n02979186 - 11971.
\item n03028079 - 10191.
\item n03888257 - 19580.
\item n03445777 - 1390.
\item n03425413 - 13862.
\item n03417042 - 18582.
\item ILSVRC2012-00045940.
\item n01440764 - 1561.
\item n02102040 - 4090.
\item n03394916 - 33380.
\item n02979186 - 2002.
\item n03028079 - 9682.
\item n03888257 - 9770.
\item n03445777 - 11162.
\item n03425413 - 15321.
\item n03417042 - 14000.
\item n03000684 - 15441.
\item n01440764 - 16051.
\item n02102040 - 6552.
\item n03394916 - 59361.
\item n02979186 - 9811.
\item n03028079 - 29942.
\item n03888257 - 20352.
\item n03445777 - 2611.
\item n03425413 - 11180.
\item n03417042 - 26782.
\item n03000684 - 7222.
\item n01440764 - 19302.
\item ILSVRC2012-00036282.
\item n03394916 - 62451.
\item n02979186 - 140.
\item n03028079 - 49281.
\item n03888257 - 14530.
\item n03445777 - 5240.
\item n03425413 - 21730.
\item n03417042 - 12790.
\item n03000684 - 13402.
\item n01440764 - 4360.
\item n02102040 - 352.
\item n03394916 - 46672.
\item n02979186 - 1542.
\item n03028079 - 15392.
\item n03888257 - 10680.
\item n03445777 - 17492.
\item n03425413 - 16220.
\item n03417042 - 7080.
\item n03000684 - 16291.
\item n01440764 - 2921.
\item n02102040 - 8061.
\item n03394916 - 30072.
\item n02979186 - 5321.
\item n03028079 - 17690.
\item n03888257 - 70632.
\item n03445777 - 9572.
\item n03425413 - 1672.
\item n03417042 - 4761.
\item n03000684 - 18591.
\item n01440764 - 8030.
\item n02102040 - 5641.
\item n03394916 - 50730.
\item n02979186 - 8861.
\item ILSVRC2012-00016542.
\item n03888257 - 3651.
\item n03445777 - 5312.
\item n03425413 - 21362.
\item n03417042 - 8822.
\item n03000684 - 19272.
\item n01440764 - 6421.
\item n02102040 - 960.
\item n03394916 - 26422.
\item n02979186 - 3260.
\item n03028079 - 6110.
\item n03888257 - 33021.
\item n03445777 - 15810.
\item n03425413 - 8661.
\item n03417042 - 21361.
\item n03000684 - 2820.
\item n01440764 - 650.
\item n02102040 - 1791.
\item n03394916 - 52191.
\item n02979186 - 14630.
\item n03028079 - 6722.
\item n03888257 - 142.
\item n03445777 - 11150.
\item n03425413 - 20500.
\item n03417042 - 27630.
\item n03000684 - 15521.
\item n01440764 - 6130.
\item n02102040 - 491.
\item n03394916 - 71910.
\item n02979186 - 8092.
\item n03028079 - 5942.
\item n03888257 - 11222.
\item n03445777 - 2530.
\item n03425413 - 602.
\item n03417042 - 5221.
\item n03000684 - 1970.
\item n01440764 - 13702.
\item n02102040 - 3450.
\item n03394916 - 35320.
\item n02979186 - 16142.
\item n03028079 - 14992.
\item n03888257 - 37362.
\item n03445777 - 6042.
\item n03425413 - 12712.
\item n03417042 - 26850.
\item n03000684 - 180.
\item n01440764 - 12881.
\item n02102040 - 4111.
\item n03394916 - 16601.
\item n02979186 - 4511.
\item n03028079 - 5432.
\item n03888257 - 64711.
\item n03445777 - 7711.
\item n03425413 - 17212.
\item n03417042 - 5510.
\item n03000684 - 19890.
\item n01440764 - 27422.
\item n02102040 - 651.
\item n03394916 - 54570.
\item n02979186 - 11.
\item n03028079 - 16731.
\item n03888257 - 13410.
\item n03445777 - 7090.
\item n03425413 - 13970.
\item n03417042 - 27862.
\item n03000684 - 2340.
\item n01440764 - 3782.
\item n02102040 - 290.
\item n03394916 - 59430.
\item n02979186 - 26820.
\item n03028079 - 3600.
\item n03888257 - 12401.
\item n03445777 - 1750.
\item n03425413 - 14302.
\item n03417042 - 28552.
\item n03000684 - 11511.
\item n01440764 - 20451.
\item n02102040 - 371.
\item n03394916 - 47852.
\item n02979186 - 3472.
\item n03028079 - 8572.
\item n03888257 - 14901.
\item n03445777 - 3301.
\item n03425413 - 14510.
\item n03417042 - 2141.
\item n03000684 - 31112.
\item n01440764 - 2102.
\item n02102040 - 2572.
\item n03394916 - 38680.
\item n02979186 - 1200.
\item n03028079 - 17922.
\item n03888257 - 15382.
\item n03445777 - 13462.
\item n03425413 - 20121.
\item n03417042 - 15592.
\item n03000684 - 31721.
\item n01440764 - 32420.
\item n02102040 - 1830.
\item n03394916 - 35811.
\item n02979186 - 12072.
\item n03028079 - 46322.
\item n03888257 - 28581.
\item n03445777 - 602.
\item n03425413 - 32871.
\item n03417042 - 18042.
\item n03000684 - 6220.
\item n01440764 - 17501.
\item n02102040 - 7392.
\item n03394916 - 36361.
\item n02979186 - 22761.
\item n03028079 - 24471.
\item n03888257 - 13790.
\item n03445777 - 7930.
\item n03425413 - 21040.
\item n03417042 - 1330.
\item n03000684 - 1542.
\item n01440764 - 8302.
\item n02102040 - 6081.
\item n03394916 - 27071.
\item n02979186 - 5781.
\item ILSVRC2012-00034021.
\item n03888257 - 38102.
\item n03445777 - 16321.
\item n03425413 - 20562.
\item n03417042 - 4560.
\item n03000684 - 6471.
\item n01440764 - 762.
\item n02102040 - 2110.
\item n03394916 - 44882.
\item n02979186 - 5481.
\item n03028079 - 9220.
\item n03888257 - 19211.
\item n03445777 - 14301.
\item n03425413 - 19050.
\item n03417042 - 6691.
\item n03000684 - 2972.
\item n01440764 - 10040.
\item n02102040 - 430.
\item n03394916 - 46391.
\item n02979186 - 13281.
\item n03028079 - 16820.
\item n03888257 - 30712.
\item n03445777 - 14232.
\item n03425413 - 21562.
\item n03417042 - 29412.
\item n03000684 - 13182.
\item n01440764 - 10852.
\item n02102040 - 5101.
\item n03394916 - 29940.
\item n02979186 - 2841.
\item n03028079 - 23280.
\item n03888257 - 23192.
\item n03445777 - 2041.
\item n03425413 - 14570.
\item n03417042 - 20280.
\item n03000684 - 8411.
\item n01440764 - 7492.
\item n02102040 - 6532.
\item n03394916 - 28590.
\item n02979186 - 560.
\item n03028079 - 38692.
\item n03888257 - 23571.
\item n03445777 - 13680.
\item ILSVRC2012-00000732.
\item n03417042 - 18551.
\item n03000684 - 34440.
\item n01440764 - 522.
\item ILSVRC2012-00008162.
\item n03394916 - 1091.
\item n02979186 - 10151.
\item n03028079 - 12802.
\item n03888257 - 171.
\item n03445777 - 7670.
\item n03425413 - 21202.
\item n03417042 - 9601.
\item ILSVRC2012-00029211.
\item n01440764 - 5432.
\item n02102040 - 4732.
\item n03394916 - 292.
\item n02979186 - 5460.
\item n03028079 - 3700.
\item n03888257 - 35800.
\item n03445777 - 9921.
\item n03425413 - 21911.
\item n03417042 - 5090.
\item n03000684 - 19211.
\item n01440764 - 8601.
\item n02102040 - 7841.
\item n03394916 - 27932.
\item n02979186 - 3161.
\item n03028079 - 29012.
\item n03888257 - 17340.
\item n03445777 - 10782.
\item n03425413 - 11161.
\item n03417042 - 29722.
\item n03000684 - 1490.
\item n01440764 - 4962.
\item n02102040 - 7792.
\item n03394916 - 47110.
\item n02979186 - 16952.
\item n03028079 - 28242.
\item n03888257 - 29762.
\item n03445777 - 230.
\item n03425413 - 3021.
\item n03417042 - 10462.
\item n03000684 - 2060.
\item n01440764 - 6301.
\item n02102040 - 2480.
\item n03394916 - 44580.
\item n02979186 - 20362.
\item n03028079 - 3492.
\item n03888257 - 30412.
\item n03445777 - 13831.
\item n03425413 - 20301.
\item n03417042 - 10280.
\item n03000684 - 16861.
\item n01440764 - 9212.
\item n02102040 - 6152.
\item n03394916 - 34332.
\item n02979186 - 14251.
\item n03028079 - 9320.
\item n03888257 - 35890.
\item n03445777 - 5131.
\item n03425413 - 16221.
\item n03417042 - 5381.
\item n03000684 - 3470.
\item n01440764 - 8142.
\item n02102040 - 762.
\item n03394916 - 51071.
\item n02979186 - 20160.
\item n03028079 - 25542.
\item n03888257 - 57010.
\item n03445777 - 261.
\item n03425413 - 7731.
\item n03417042 - 3821.
\item ILSVRC2012-00047060.
\item n01440764 - 12971.
\item n02102040 - 1300.
\item n03394916 - 7292.
\item n02979186 - 23362.
\item n03028079 - 10020.
\item ILSVRC2012-00038942.
\item n03445777 - 11690.
\item n03425413 - 13100.
\item n03417042 - 6811.
\item n03000684 - 20762.
\item n01440764 - 11350.
\item n02102040 - 1822.
\item n03394916 - 33012.
\item n02979186 - 1061.
\item n03028079 - 16660.
\item n03888257 - 38200.
\item n03445777 - 10671.
\item n03425413 - 6772.
\item n03417042 - 1492.
\item n03000684 - 24991.
\item n01440764 - 7462.
\item n02102040 - 362.
\item n03394916 - 26802.
\item n02979186 - 3530.
\item n03028079 - 80.
\item n03888257 - 66102.
\item n03445777 - 8192.
\item ILSVRC2012-00035211.
\item n03417042 - 10300.
\item n03000684 - 16072.
\item n01440764 - 8451.
\item n02102040 - 3260.
\item ILSVRC2012-00025761.
\item n02979186 - 5031.
\item n03028079 - 10241.
\item n03888257 - 12400.
\item n03445777 - 6201.
\item n03425413 - 260.
\item n03417042 - 2062.
\item n03000684 - 27850.
\item n01440764 - 9491.
\item n02102040 - 821.
\item n03394916 - 32340.
\item n02979186 - 1932.
\item n03028079 - 26291.
\item n03888257 - 9552.
\item n03445777 - 101.
\item n03425413 - 1792.
\item n03417042 - 18152.
\item n03000684 - 5041.
\item n01440764 - 4980.
\item n02102040 - 3532.
\item n03394916 - 6742.
\item n02979186 - 22882.
\item n03028079 - 25462.
\item n03888257 - 8381.
\item n03445777 - 5382.
\item n03425413 - 13232.
\item n03417042 - 9170.
\item n03000684 - 17330.
\item n01440764 - 6361.
\item n02102040 - 142.
\item n03394916 - 51161.
\item n02979186 - 15931.
\item n03028079 - 28662.
\item n03888257 - 12070.
\item n03445777 - 10401.
\item n03425413 - 4511.
\item n03417042 - 1601.
\item n03000684 - 10212.
\item n01440764 - 7752.
\item n02102040 - 1110.
\item n03394916 - 38212.
\item n02979186 - 1621.
\item n03028079 - 2060.
\item n03888257 - 7610.
\item n03445777 - 7902.
\item n03425413 - 21211.
\item n03417042 - 3390.
\item n03000684 - 11821.
\item n01440764 - 8221.
\item n02102040 - 350.
\item n03394916 - 37321.
\item n02979186 - 2312.
\item n03028079 - 16501.
\item n03888257 - 11081.
\item n03445777 - 471.
\item n03425413 - 24461.
\item n03417042 - 6272.
\item n03000684 - 6460.
\item n01440764 - 7160.
\item n02102040 - 3112.
\item n03394916 - 33221.
\item n02979186 - 15972.
\item n03028079 - 9112.
\item n03888257 - 7130.
\item n03445777 - 8861.
\item n03425413 - 14552.
\item n03417042 - 2960.
\item n03000684 - 5231.
\item n01440764 - 16072.
\item n02102040 - 672.
\end{itemize}
\end{multicols}

\paragraph{DTD2}
The Describable Textures Dataset contains 47 classes, grouped according to texture: bumpy, dotted, lined, veined etc. The image ids below are of the form <class-id> - <index-within-class>. Unsurprisingly, most of the suitable images, i.e., those with detailed repeating textures, are from classes such as `woven', `grid', or `wrinkled'. The number of images of each class is given below.

\begin{figure}[h]
    \centering
    \includegraphics[width=\textwidth]{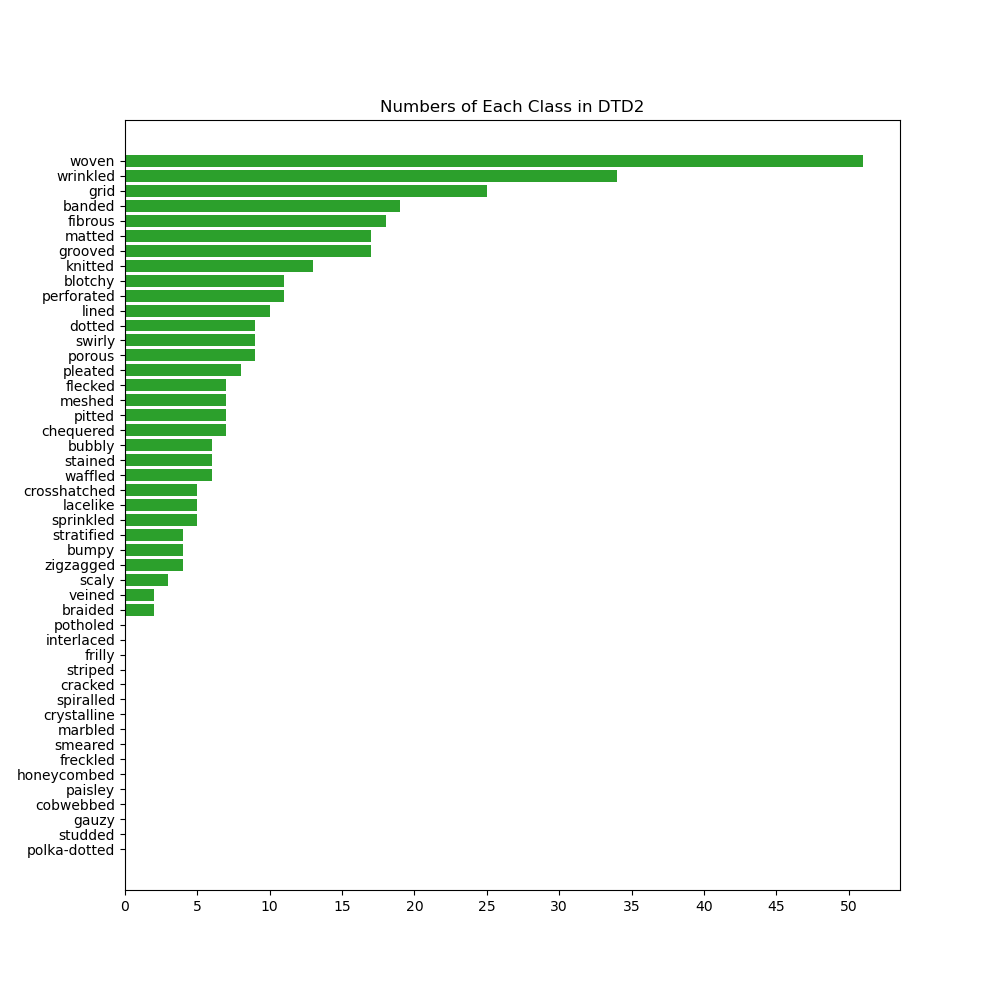}
    \vspace{-16pt}
    \caption{Number of images of each class from DTD that we select to be part of our curated dataset, DTD2. We select images that show a detailed, repeating pattern, meaning some types of textures are much more likely to be selected.}
    \label{fig:dtd2-class-counts}
\end{figure}

\begin{multicols}{2}
\begin{itemize}
\item perforated - 0074.
\item meshed - 0164.
\item dotted - 0164.
\item sprinkled - 0066.
\item porous - 0117.
\item woven - 0106.
\item knitted - 0185.
\item crosshatched - 0081.
\item pleated - 0163.
\item banded - 0046.
\item wrinkled - 0129.
\item banded - 0055.
\item braided - 0008.
\item grooved - 0119.
\item dotted - 0131.
\item wrinkled - 0015.
\item bumpy - 0098.
\item woven - 0004.
\item zigzagged - 0109.
\item matted - 0136.
\item stratified - 0174.
\item grooved - 0051.
\item perforated - 0014.
\item grooved - 0089.
\item woven - 0025.
\item wrinkled - 0106.
\item lined - 0159.
\item banded - 0008.
\item matted - 0070.
\item lined - 0109.
\item dotted - 0192.
\item fibrous - 0138.
\item matted - 0150.
\item pleated - 0142.
\item grid - 0088.
\item blotchy - 0038.
\item chequered - 0043.
\item banded - 0081.
\item wrinkled - 0063.
\item waffled - 0156.
\item grid - 0073.
\item grid - 0016.
\item lacelike - 0078.
\item matted - 0071.
\item chequered - 0050.
\item wrinkled - 0132.
\item porous - 0149.
\item stained - 0119.
\item knitted - 0118.
\item pleated - 0116.
\item stained - 0066.
\item knitted - 0144.
\item chequered - 0062.
\item grooved - 0085.
\item blotchy - 0091.
\item knitted - 0150.
\item grid - 0124.
\item pleated - 0168.
\item zigzagged - 0133.
\item grooved - 0058.
\item zigzagged - 0008.
\item stained - 0132.
\item blotchy - 0088.
\item bumpy - 0067.
\item grid - 0049.
\item woven - 0062.
\item blotchy - 0059.
\item matted - 0069.
\item lined - 0133.
\item woven - 0075.
\item bubbly - 0097.
\item matted - 0073.
\item porous - 0151.
\item blotchy - 0083.
\item chequered - 0052.
\item wrinkled - 0041.
\item lacelike - 0096.
\item matted - 0085.
\item fibrous - 0150.
\item banded - 0122.
\item waffled - 0124.
\item fibrous - 0164.
\item grid - 0083.
\item fibrous - 0193.
\item dotted - 0060.
\item meshed - 0176.
\item woven - 0043.
\item woven - 0088.
\item stained - 0090.
\item wrinkled - 0045.
\item pleated - 0090.
\item zigzagged - 0085.
\item veined - 0135.
\item dotted - 0132.
\item stratified - 0046.
\item woven - 0061.
\item woven - 0028.
\item swirly - 0074.
\item matted - 0065.
\item sprinkled - 0065.
\item waffled - 0068.
\item grooved - 0048.
\item perforated - 0066.
\item grid - 0022.
\item woven - 0053.
\item porous - 0152.
\item fibrous - 0160.
\item woven - 0055.
\item matted - 0148.
\item pitted - 0134.
\item flecked - 0060.
\item lacelike - 0020.
\item grid - 0052.
\item woven - 0067.
\item knitted - 0192.
\item flecked - 0053.
\item chequered - 0054.
\item chequered - 0088.
\item lined - 0027.
\item stained - 0030.
\item knitted - 0146.
\item grid - 0078.
\item blotchy - 0070.
\item swirly - 0060.
\item perforated - 0057.
\item porous - 0098.
\item wrinkled - 0087.
\item blotchy - 0096.
\item grooved - 0081.
\item wrinkled - 0039.
\item lined - 0041.
\item flecked - 0126.
\item lined - 0038.
\item dotted - 0135.
\item pitted - 0036.
\item wrinkled - 0103.
\item wrinkled - 0034.
\item grid - 0050.
\item bubbly - 0083.
\item woven - 0001.
\item knitted - 0116.
\item pleated - 0069.
\item wrinkled - 0043.
\item woven - 0059.
\item knitted - 0079.
\item matted - 0128.
\item lacelike - 0017.
\item fibrous - 0165.
\item wrinkled - 0088.
\item grid - 0129.
\item blotchy - 0090.
\item wrinkled - 0017.
\item sprinkled - 0038.
\item woven - 0032.
\item flecked - 0074.
\item woven - 0029.
\item knitted - 0130.
\item crosshatched - 0092.
\item lacelike - 0065.
\item knitted - 0141.
\item grid - 0011.
\item porous - 0099.
\item woven - 0039.
\item woven - 0113.
\item fibrous - 0211.
\item sprinkled - 0067.
\item wrinkled - 0125.
\item crosshatched - 0093.
\item dotted - 0154.
\item woven - 0130.
\item veined - 0075.
\item meshed - 0181.
\item fibrous - 0103.
\item fibrous - 0183.
\item woven - 0082.
\item woven - 0099.
\item perforated - 0041.
\item grid - 0099.
\item grooved - 0084.
\item meshed - 0162.
\item wrinkled - 0036.
\item banded - 0147.
\item porous - 0157.
\item wrinkled - 0108.
\item dotted - 0185.
\item grid - 0089.
\item grid - 0101.
\item woven - 0048.
\item grid - 0066.
\item bumpy - 0190.
\item matted - 0166.
\item woven - 0104.
\item waffled - 0171.
\item wrinkled - 0040.
\item flecked - 0135.
\item swirly - 0151.
\item stratified - 0115.
\item perforated - 0045.
\item woven - 0026.
\item fibrous - 0111.
\item swirly - 0065.
\item perforated - 0026.
\item banded - 0107.
\item woven - 0068.
\item banded - 0037.
\item fibrous - 0204.
\item wrinkled - 0026.
\item waffled - 0178.
\item woven - 0108.
\item grooved - 0164.
\item woven - 0021.
\item fibrous - 0127.
\item banded - 0141.
\item scaly - 0131.
\item woven - 0123.
\item braided - 0167.
\item woven - 0046.
\item grooved - 0057.
\item perforated - 0024.
\item swirly - 0137.
\item grid - 0081.
\item bubbly - 0118.
\item grooved - 0108.
\item wrinkled - 0079.
\item flecked - 0003.
\item fibrous - 0120.
\item wrinkled - 0114.
\item woven - 0083.
\item fibrous - 0110.
\item wrinkled - 0067.
\item lined - 0169.
\item wrinkled - 0025.
\item wrinkled - 0021.
\item wrinkled - 0013.
\item dotted - 0041.
\item woven - 0049.
\item lined - 0076.
\item scaly - 0122.
\item grid - 0059.
\item waffled - 0081.
\item matted - 0117.
\item fibrous - 0101.
\item stained - 0075.
\item woven - 0036.
\item wrinkled - 0086.
\item wrinkled - 0084.
\item banded - 0047.
\item banded - 0068.
\item matted - 0155.
\item perforated - 0080.
\item pitted - 0078.
\item pitted - 0008.
\item fibrous - 0089.
\item sprinkled - 0068.
\item woven - 0084.
\item grooved - 0093.
\item woven - 0056.
\item pitted - 0064.
\item wrinkled - 0046.
\item woven - 0109.
\item banded - 0086.
\item grid - 0084.
\item grid - 0116.
\item woven - 0038.
\item pleated - 0082.
\item bumpy - 0140.
\item wrinkled - 0111.
\item matted - 0115.
\item fibrous - 0096.
\item woven - 0093.
\item swirly - 0144.
\item banded - 0059.
\item lined - 0141.
\item woven - 0003.
\item banded - 0114.
\item woven - 0002.
\item woven - 0127.
\item knitted - 0126.
\item banded - 0115.
\item woven - 0051.
\item lined - 0166.
\item bubbly - 0084.
\item flecked - 0165.
\item wrinkled - 0105.
\item woven - 0114.
\item woven - 0126.
\item grooved - 0063.
\item wrinkled - 0083.
\item grooved - 0088.
\item grid - 0032.
\item wrinkled - 0065.
\item grid - 0067.
\item perforated - 0016.
\item meshed - 0108.
\item blotchy - 0082.
\item grooved - 0045.
\item swirly - 0147.
\item grooved - 0068.
\item woven - 0065.
\item grid - 0085.
\item blotchy - 0089.
\item wrinkled - 0085.
\item woven - 0107.
\item stratified - 0100.
\item fibrous - 0201.
\item scaly - 0137.
\item woven - 0071.
\item perforated - 0119.
\item swirly - 0138.
\item grooved - 0083.
\item matted - 0072.
\item grid - 0093.
\item chequered - 0093.
\item knitted - 0098.
\item matted - 0084.
\item pitted - 0010.
\item porous - 0053.
\item matted - 0129.
\item woven - 0066.
\item crosshatched - 0109.
\item wrinkled - 0033.
\item pitted - 0157.
\item porous - 0142.
\item woven - 0092.
\item crosshatched - 0116.
\item swirly - 0159.
\item grid - 0082.
\item banded - 0099.
\item meshed - 0112.
\item knitted - 0155.
\item woven - 0112.
\item meshed - 0161.
\item banded - 0061.
\item bubbly - 0055.
\item woven - 0063.
\item bubbly - 0096.
\item blotchy - 0041.
\item banded - 0002.
\item woven - 0007.
\item banded - 0090.
\item pleated - 0094.
\end{itemize}
\end{multicols}
\paragraph{CIFAR10 and MNIST}
For CIFAR and MNIST, we report the distribution of classes in our random sample of 500, in Tables \ref{tab:cifar-counts} and \ref{tab:mnist-counts}, respectively.

\begin{table}
    \caption{Counts of each class in our random sample of 500 CIFAR images.}
    
    \smallskip 
    \label{tab:cifar-counts}
    \centering
    \begin{tabular}{c|*{10}{c}}
    \toprule
        \textbf{class} & airplane & car & bird & cat & deer & dog & frog & horse & ship & truck \\
         \midrule
        \textbf{count} &  66 & 66 & 48 & 48 & 48 & 47 & 49 & 40 & 48 & 40 \\
        \bottomrule
    \end{tabular}
    \end{table}

\begin{table}
    \caption{Counts of each class in our random sample of 500 MNIST images.}
    
    \smallskip 
    \label{tab:mnist-counts}
    \centering
    \begin{tabular}{c|*{10}{c}}
    \toprule
           \textbf{class} & 0 & 1 & 2 & 3 & 4 & 5 & 6 & 7 & 8 & 9 \\
         \midrule
        \textbf{count} &  47 & 61 & 57 & 52 & 47 & 45 & 41 & 50 & 53 & 47 \\
        \bottomrule
    \end{tabular}
    \end{table}

\end{document}